\documentclass{article} % For LaTeX2e
\usepackage{iclr2025_conference,times}

\usepackage{amsmath,amsfonts,bm}

\def\eqref#1{equation~\ref{#1}}
\def\1{\bm{1}}

\DeclareMathAlphabet{\mathsfit}{\encodingdefault}{\sfdefault}{m}{sl}
\SetMathAlphabet{\mathsfit}{bold}{\encodingdefault}{\sfdefault}{bx}{n}

\usepackage{hyperref}
\usepackage{url}

\usepackage{microtype}
\usepackage{graphicx}
\usepackage{subfigure}
\usepackage{booktabs} % for professional tables
\usepackage{algorithmic}
\usepackage{algorithm}

\usepackage{amsmath}
\usepackage{amssymb}
\usepackage{mathtools}
\usepackage{amsthm}

 \usepackage{tabularray}
\theoremstyle{plain}
\newtheorem{theorem}{Theorem}[section]
\newtheorem{proposition}[theorem]{Proposition}
\newtheorem{lemma}[theorem]{Lemma}

\theoremstyle{definition}

\theoremstyle{remark}

\newcommand{\highlight}[1]{{\color{blue}#1}}

\newcommand{\cA}{{\mathcal{A}}}

\newcommand{\cD}{{\mathcal{D}}}
\newcommand{\cS}{{\mathcal{S}}}
\newcommand{\cM}{{\mathcal{M}}}

\newcommand{\cL}{{\mathcal{L}}}

\newcommand{\cN}{{\mathcal{N}}}

\newcommand{\cO}{{\mathcal{O}}}

\newcommand{\cZ}{{\mathcal{Z}}}

\newcommand{\bX}{\textbf{X}}

\newcommand{\bs}{\textbf{s}}

\newcommand{\bq}{\textbf{q}}

\newcommand{\bA}{\textbf{A}}

\newcommand{\ba}{\textbf{a}}

\newcommand{\bo}{\textbf{o}}

\newcommand{\bpi}{\pmb{\pi}}
\newcommand{\bnu}{\pmb{\nu}}

\newcommand{\bbR}{\mathbb{R}}
\newcommand{\bbE}{\mathbb{E}}

\newcommand{\dtot}{\rho^{\pmb{\pi}_{tot}}}
\newcommand{\dmutot}{\rho^{\mu_{tot}}}

\usepackage{mathtools}

\usepackage{graphicx}
\usepackage{subcaption}
\usepackage{multirow}
\usepackage{color, colortbl}
\usepackage{rotating}
\usepackage{setspace}
\usepackage{wrapfig}

\newcommand{\squishend}{
  \end{list}  }

\newif\ifnotes\notestrue
\def\cmt#1{{\color{blue}{#1}}}

\def\htien#1{}

\newcommand{\red}[1]{{\textbf{#1}}}
\usepackage{mathtools}

\newif\ifnotes\notestrue
\def\htien#1{}

\usepackage{xcolor}

\newcommand{\showinstance}[4][0.25]{
\begin{minipage}{#1\textwidth}
    \centering
    \includegraphics[width=1.0\textwidth,trim={#2 0 0 0},clip]{graphs/#3.pdf}
    \ifx #4\empty \else \parbox{\linewidth}{\centering #4} \fi
    \captionsetup{justification=centering}
\end{minipage}
}

\newcommand{\showlegend}[3][0.25]{
\begin{minipage}{#1\textwidth}
    \centering
    \includegraphics[width=1.0\textwidth,trim={#2cm #3cm 0 0},clip]{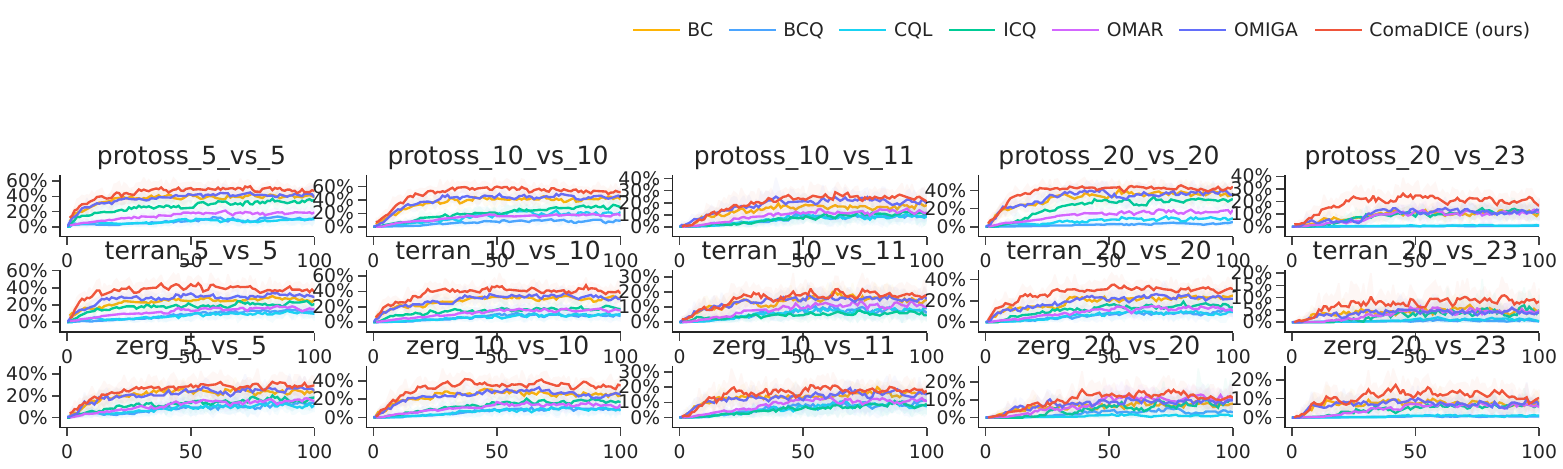}
\end{minipage}
}

\title{ComaDICE: Offline Cooperative Multi-Agent Reinforcement Learning with Stationary Distribution Shift Regularization}

\usepackage{fancyhdr}
\fancypagestyle{plain}{
    \fancyhf{}
    \fancyhead[C]{}
    \renewcommand{\headrulewidth}{0pt}
}
\usepackage[a4paper, margin=1in]{geometry}  % Set equal margins

\author{The Viet Bui \\
Singapore Management University, Singapore\\
\texttt{theviet.bui.2023@phdcs.smu.edu.sg}
\And
Thanh Hong Nguyen\\
University of Oregon Eugene, Oregon, United States\\ 
\texttt{thanhhng@cs.uoregon.edu}
\AND
Tien Mai \\
Singapore Management University, Singapore\\
\texttt{atmai@smu.edu.sg}
}

\iclrfinalcopy % Uncomment for camera-ready version, but NOT for submission.
\begin{document}

\addtocontents{toc}{\protect\setcounter{tocdepth}{0}}
\maketitle

\begin{abstract}
Offline reinforcement learning (RL) has garnered significant attention for its ability to learn effective policies from pre-collected datasets without the need for further environmental interactions. While promising results have been demonstrated in single-agent settings, offline multi-agent reinforcement learning (MARL) presents additional challenges due to the large joint state-action space and the complexity of multi-agent behaviors. A key issue in offline RL is the \textit{distributional shift}, which arises when the target policy being optimized deviates from the behavior policy that generated the data. This problem is exacerbated in MARL due to the interdependence between agents' local policies and the expansive joint state-action space. Prior approaches have primarily addressed this challenge by incorporating regularization in the space of either Q-functions or policies. In this work, we introduce a regularizer in the space of stationary distributions to better handle distributional shift. Our algorithm, ComaDICE, offers a principled framework for offline cooperative MARL by incorporating stationary distribution regularization for the global learning policy, complemented by a carefully structured multi-agent value decomposition strategy to facilitate  multi-agent training. Through extensive experiments on the multi-agent \textit{MuJoCo} and \textit{StarCraft II}  benchmarks, we demonstrate that ComaDICE achieves superior performance compared to state-of-the-art offline MARL methods across nearly all tasks.

\end{abstract}

\section{Introduction}
Over the years, deep RL has achieved remarkable success in various decision-making tasks~\citep{levine2016end,silver2017mastering,kalashnikov2018scalable,haydari2020deep}. However, a significant limitation of deep RL is its need for millions of interactions with the environment to gather experiences for policy improvement. This process can be both costly and risky, especially in real-world applications like robotics and healthcare.
To address this challenge, offline RL has emerged, enabling policy learning based solely on pre-collected demonstrations~\citep{levine2020offline}. Despite this advancement, offline RL faces a critical issue: the distribution shift between the offline dataset and the learned policy~\citep{kumar2019stabilizing}. This distribution shift complicates value estimation for unseen states and actions during policy evaluation, resulting in {extrapolation errors} where out-of-distribution (OOD) state-action pairs are assigned unrealistic values~\citep{fujimoto2018addressing}. 

To tackle OOD actions, many existing works impose action-level constraints, either implicitly by regulating the learned value functions or explicitly through distance or divergence penalties~\citep{fujimoto2019off,kumar2019stabilizing,wu2019behavior,peng2019advantage,fujimoto2021minimalist,xu2021offline}. Only a few recent studies have addressed both OOD actions and states using state-action-level behavior constraints~\citep{li2022dealing,zhang2022state,lee2021optidice,lee2022coptidice,mao2024odice}. In particular, there is an important line of work on DIstribution Correction Estimation (DICE)~\citep{nachum2020reinforcement} that constrains the distance in terms of the joint state-action occupancy measure between the learning policy and the offline policy. These DICE-based methods have demonstrated impressive performance results on the D4RL benchmarks~\citep{lee2021optidice,lee2022coptidice,mao2024odice}.  

It is important to note that that all the aforementioned offline RL approaches primarily focus on the single-agent setting. While multi-agent setting is prevalent in many real-world sequential decision-making tasks, offline MARL remains a relatively under-explored area. The multi-agent setting poses significantly greater challenges due to the large joint state-action space, which expands exponentially with the number of agents, as well as the inter-dependencies among the local policies of different agents. As a result, the offline data distribution can become quite sparse in these high-dimensional joint action spaces, leading to an increased number of OOD state-action pairs and exacerbating extrapolation errors.
A few recent studies have sought to address the negative effects of sparse data distribution in offline MARL by adapting the well-known centralized training decentralized execution (CTDE) paradigm from online MARL~\citep{oliehoek2008optimal,kraemer2016multi}, enabling data-related regularization at the individual agent level. Notably, some of these works~\citep{pan2022plan,shao2024counterfactual,wang2024offline_OMIGA} extend popular offline single-agent RL algorithms, such as CQL~\citep{kumar2020conservative} and SQL/EQL~\citep{xu2023offline}, within the CTDE framework. 
% To the best of our knowledge, most existing works in offline MARL primarily focus on the action-level regularization.

In our work, we focus on addressing the aforementioned challenges in offline cooperative MARL. In particular, we follow the DICE approach to address both OOD states and actions, motivated by remarkable performance of recent DICE-based methods in offline single-agent RL. Similar to previous works in offline MARL, we adopt the CTDE framework to handle exponential joint state-action spaces in the multi-agent setting. We remark that extending the DICE approach under this CTDE framework is not straightforward given the complex objective of DICE that involves the f-divergence in stationary distribution between the learning joint policy and the behavior policy. Therefore, the value decomposition in CTDE needs to be carefully designed to ensure the consistency in optimality between the global and local policies. In particular, we provide the following main contributions:
\begin{itemize}
    \item We propose ComaDICE, a new offline MARL algorithm that integrates DICE with a carefully designed value decomposition strategy. In ComaDICE, under the CTDE framework, we decompose both the global value function $\nu^{tot}$ and the global advantage functions $A^{tot}_{\nu}$, rather than using Q-functions as in previous MARL works.
    This unique factorization approach allows us to theoretically demonstrate that the global learning objective in DICE is convex in local values, provided that the mixing network used in the value decomposition employs non-negative weights and convex activation functions. This significant finding ensures that our decomposition strategy promotes an efficient and stable training process.
    \item Building on our decomposition strategy, we demonstrate that finding an optimal global policy can be divided into multiple sub-problems, each aims to identify a local optimal policy for an individual agent. We provide a theoretical proof that the global optimal policy is, in fact, equivalent to the product of the local policies derived from these sub-problems.
    \item Finally, we conduct extensive experiments to evaluate the performance of our algorithm, ComaDICE, in complex MARL environments, including: multi-agent StarCraft II (i.e., SMACv1~\citep{samvelyan2019starcraft}, SMACv2~\citep{ellis2022smacv2}) and multi-agent Mujoco~\citep{de2020deep} benchmarks. Our empirical results show that our ComaDICE outperforms several strong baselines in all these benchmarks.
\end{itemize}

\section{Related Work}
\paragraph{Offline Reinforcement Learning (offline RL).} 
Offline RL focuses on learning policies from pre-collected datasets without any further interactions with the environment~\citep{levine2020offline,prudencio2023survey}. A significant challenge in offline RL is the issue of distribution shift, where unseen actions and states may arise during training and execution, leading to inaccurate policy evaluations and suboptimal outcomes. Consequently, there is a substantial body of literature addressing this challenge through various approaches~\citep{prudencio2023survey}.
In particular, some studies impose explicit or implicit policy constraints to ensure that the learned policy remains close to the behavioral policy~\citep{fujimoto2019off,kumar2019stabilizing,wu2019behavior,kostrikov2021offline,peng2019advantage,nair2020awac,fujimoto2021minimalist,xu2021offline,cheng2024look,li2023proto}. Others incorporate regularization terms into the learning objectives to mitigate the value overestimation on OOD actions~\citep{kumar2020conservative,kostrikov2021offline,xu2022constraints,niu2022trust,xu2023offline,wang2024offline_OMIGA}.
Uncertainty-based offline RL methods seek to balance conservative approaches with naive off-policy RL techniques, relying on estimates of model, value, or policy uncertainty~\citep{agarwal2020optimistic,an2021uncertainty,bai2022pessimistic}. Offline model-based algorithms focus on conservatively estimating the transition dynamics and reward functions based on the pre-collected datasets~\citep{kidambi2020morel,yu2020mopo,matsushima2020deployment,yu2021combo}. Some other methods impose action-level regularization through imitation learning techniques ~\citep{xu2022discriminator,chen2020bail,zhang2023saformer,zheng2024safe,brandfonbrener2021offline,xu2022policy}. Finally, while a majority of previous works target OOD actions only, there are a few recent works attempt to address both OOD states and actions~\citep{li2022dealing,zhang2022state,lee2021optidice,lee2022coptidice,sikchi2023imitation,mao2024odice}. Our work on offline MARL follow the DICE-based approach, as motivated by compelling performance of DICE-based algorithms in single-agent settings~\citep{lee2021optidice,lee2022coptidice,sikchi2023imitation,mao2024odice}.

% Some works impose policy constraints either implicitly or explicitly, ensuring the behavioral policy and the learned policy are close to each others.

% Some other works, on the other hand, 
% For example, some existing works incorporate penalties or constraints on the divergence between the learning policy and the behavior policy~\cite{wu2019behavior,kumar2019stabilizing,xu2021offline,fujimoto2021minimalist,li2023proto,cheng2024look}. Some other works, on the other hand, add some form of regularization into the learning objective such as incorporating pessimistic evaluation terms on OOD actions~\cite{kumar2020conservative,kostrikov2021offline,xu2022constraints,niu2022trust,xu2023offline,wang2024offline}     
\paragraph{Offline Multi-agent Reinforcement Learning (offline MARL).} While there is a substantial body of literature on offline single-agent RL, research on offline MARL remains limited. Offline MARL faces challenges from both distribution shift—characteristic of offline settings—and the exponentially large joint action space typical of multi-agent environments. Recent studies have begun to merge advanced methodologies from both offline RL and MARL to address these challenges~\citep{yang2021believe,pan2022plan,shao2024counterfactual,wang2024offline_OMIGA}
Specifically, these works employ local policy regularization within the centralized training with decentralized execution (CTDE) framework to mitigate distribution shift. The CTDE paradigm, well-established in online MARL, facilitates more efficient and stable learning while allowing agents to operate in a decentralized manner~\citep{oliehoek2008optimal,kraemer2016multi}. For instance, \citet{yang2021believe} utilize importance sampling to manage local policy learning on OOD samples. Both works by~\citet{pan2022plan} and~\citet{shao2024counterfactual} are built upon CQL~\citep{kumar2020conservative}, a prominent offline RL algorithm for single-agent scenarios. \citet{matsunaga2023alberdice} adopt the Nash equilibrium solution concept in game theory to iteratively update best responses of individual agents. Finally, OMIGA~\citep{wang2024offline_OMIGA} establishes the equivalence between global and local value regularization within a policy constraint framework, making it the current state-of-the-art algorithm in offline MARL.

Beyond this main line of research, some studies formulate offline MARL as a sequence modeling problem, employing supervised learning techniques to tackle the issue~\citep{meng2023offline,tseng2022offline}, while others adhere to decentralized approaches~\citep{jiang2023offline}.

%\section{{Background}}

 \section{{Preliminaries}\label{sec:background}
}

 Our work focuses on  cooperative multi-agent RL, which can be modeled as a multi-agent Partially Observable Markov Decision Process (POMDP), defined by the tuple $\cM = \langle \cS, \cA, P, r, \cZ, \cO, n,\cN, \gamma \rangle$. Here, $n$ is number of agents, $\cN = \{1,\ldots,n\}$ is the set of agents,  $\bs \in S$ represents the true state of the multi-agent environment, and $\cA =  \prod_{i\in \cN} \cA_i $ is the set of joint  actions, where $\cA_i$ is the set of individual actions  available to agent $i\in \cN$. At each time step, each agent $i \in \{1, 2, \dots, n\}$ selects an action $a_i \in \cA_i$, forming a joint action $\ba = (a_1, a_2, \dots, a_n) \in \cA$. The transition dynamics $P(\bs'|\bs, \ba) : \cS \times \cA \times \cS \to [0, 1]$ describe the probability of transitioning to the next  state $\bs'$ when agents take an action $\ba$ from the current state $\bs$. The discount factor $\gamma \in [0, 1)$ represents the weight given to future rewards.  In a partially observable environment, each agent receives a local observation $o_i \in \cO_i$ based on the observation function $\cZ_i(\bs) : \cS \to \cO_i$, and we denote the joint observation as $\bo = (o_1, o_2, \dots, o_n)$. In cooperative MARL, all agents share a global reward function $r(\bo, \ba) : \cO \times \cA \to \mathbb{R}$. The goal of all agents is to learn a joint policy $\bpi_{\text{tot}} = \{\pi_1, \dots, \pi_n\}$ that collectively maximize the expected discounted returns $\mathbb{E}_{(\bo,\ba) \sim \bpi_{\text{tot}}} \left[\sum_{t=0}^{\infty} \gamma^t r(\bo_t, \ba_t)\right]$. In the offline MARL setting, a pre-collected dataset $\cD$ is obtained by sampling from a behavior policy $\mu_{\text{tot}} = \{\mu_1, \dots, \mu_n\}$, and the policy learning is conducted soly based on $\cD$, with no interactions with the environment. We also define the occupancy measure (or stationary distribution) as follows:
 \[
 \rho^{\bpi_{tot}}(\bo,\ba) = (1-\gamma) \sum\nolimits_{t=0}^\infty P(\bo_t = \bo,~\ba_t = \ba)
 \]
which represents distribution visiting the pair (observation, action) $(\bo_t,\ba_1)$ when following the joint policy $\bpi_{tot}$, where $\bs_0 \sim P_0, ~\ba_t \sim \bpi_{tot}(\cdot|\bs_t)$  and $\bs_{t+1} \sim P(\cdot|\bs_t,\ba_t)$.

 % \subsection{Offline MARL with Behavior Regularization}

 % \subsection{Offline RL with Stationary Distribution Correction Estimation}
 
\section{ComaDICE: Offline Cooperative Multi-Agent RL with Stationary Distribution Correction Estimation}

We consider an offline cooperative MARL problem where the goal is to optimize the expected discounted joint reward. In this work, we focus on the DICE objective function~\cite{nachum2020reinforcement}, which incorporates a stationary distribution regularizer to capture the divergence between the occupancy measures of the learning policy, $\bpi_{tot}$, and the behavior policy, \(\mu_{tot}\), formulated as follows:
\begin{align}
    \max\nolimits_{\bpi_{tot}} ~~\bbE_{(\bo,\ba)\sim \rho^{\bpi_{tot}}}\left[r(\bo,\ba)\right]  - \alpha D^f\left(\rho^{\bpi_{tot}} \parallel \rho^{\mu_{tot}}\right)\label{prob:main-dice}
\end{align}
where
$D^f\left(\rho^{\bpi_{tot}} \parallel \rho^{\mu_{tot}}\right) = \bbE_{(\bo,\ba)\sim \rho^{\bpi_{tot}}} \left[f\left(\frac{\rho^{\bpi_{tot}}}{\rho^{\mu_{tot}}}\right)\right]$
is the f-divergence between the stationary distribution \(\rho^{\bpi_{tot}}\) of the learning policy and \(\rho^{\mu_{tot}}\) of the behavior policy. In this work, we consider $f(\cdot)$ to be strictly convex and differentiable. The parameter \(\alpha\) controls the trade-off between maximizing the reward and penalizing deviation from the offline dataset's distribution (i.e., penalizing distributional shift). When \(\alpha = 0\), the problem becomes the standard offline MARL, where the objective is to find a joint policy that maximizes the expected joint reward. On the other hand, when \(\alpha \gg 1\), the problem shifts towards imitation learning, aiming to closely mimic the behavioral policy.
% Here, unlike many previous studies on offline cooperative MARL, we impose regularization on the stationary distribution rather than on the policies. 

This DICE-based approach offers the advantage of better capturing the system dynamics inherent in the offline data. Such stationary distributions, \(\rho^{\bpi_{tot}}\) and \(\dmutot\), however, are not directly available. We will discuss how to estimate them in the next subsection.

\subsection{Constrained Optimization in the  Stationary Distribution Space}
 We first formulate the learning problem in Eq. \ref{prob:main-dice} as a constrained optimization on the space of $\dtot$:
 \begin{align}
     \max\nolimits_{\dtot}\quad & \bbE_{(\bo,\ba)\sim \rho^{\bpi_{tot}}}\left[r(\bo,\ba)\right]  - \alpha D^f\left(\rho^{\bpi_{tot}} \parallel \rho^{\mu_{tot}}\right)\label{prob:dice-2}\\
     \textbf{s.t.}\quad & \sum\nolimits_{\ba'} \dtot(\bs,\ba') = (1-\gamma)p_0(\bs) + \gamma \sum\nolimits_{\ba',\bs'}  \dtot(\bs',\ba') P(\bs| \ba',\bs'),~\forall \bs \in \cS \label{prob:dice-2.1}
 \end{align}
% Where \(\alpha > 0\) represents the weight of the regularizer, setting \(\alpha = 0\) reduces the problem to a standard offline MARL, where the objective is to find a joint policy that maximizes the expected joint reward. On the other hand, when \(\alpha \gg 1\), the problem shifts towards imitation learning, aiming to closely mimic the behavioral policy. 
When \(f\) is convex, (\ref{prob:dice-2}-\ref{prob:dice-2.1}) becomes a convex optimization problem, as it involves maximizing a concave objective function subject to linear constraints.
We now consider the Lagrange dual of (\ref{prob:dice-2}-\ref{prob:dice-2.1}):
\begin{align}
    \cL(\nu^{tot}, &\dtot) = \bbE_{(\bo, \ba) \sim \rho^{\bpi_{tot}}}\left[r(\bo, \ba)\right] 
- \alpha \bbE_{(\bs, \ba) \sim \dmutot}\left[f\left(\frac{\dtot(\bs, \ba)}{\dmutot(\bs, \ba)}\right)\right] \nonumber\\
&- \sum\nolimits_{\bs}\nu^{tot}(\bs) \Big(\sum\nolimits_{\ba'} \dtot(\bs, \ba') - (1 - \gamma) p_0(\bs) - \gamma \sum\nolimits_{\ba', \bs'} \dtot(\bs', \ba') P(\bs | \ba', \bs')\Big)
\end{align}
where $\nu^{tot}(\bs)$ is a Lagrange multiplier. Since (\ref{prob:dice-2}-\ref{prob:dice-2.1}) is a convex optimization problem, it is equivalent to the following minimax problem over the spaces of $\nu^{tot}$ and $\dtot$:
\[
\min\nolimits_{\nu^{tot}} \max\nolimits_{\dtot} \left\{ \cL(\nu^{tot}, \dtot) \right\}
\]
Furthermore, we observe that $\cL(\nu^{tot}, \dtot)$ is linear in $\nu^{tot}$ and concave in $\dtot$, so the minimax problem has a saddle point, implying:
$\min_{\nu^{tot}} \max_{\dtot} \left\{ \cL(\nu^{tot}, \dtot) \right\} = \max_{\dtot} \min_{\nu^{tot}} \left\{ \cL(\nu^{tot}, \dtot) \right\}.$
In a manner analogous to the single-agent case, by defining $w^{tot}_\nu(\bs, \ba) = \frac{\dtot(\bs, \ba)}{\dmutot(\bs, \ba)}$, the Lagrange dual function can be simplified into the more compact form (with detailed derivations are in the appendix):
\begin{align*}
    &\cL(\nu^{tot}, w^{tot}) = (1 - \gamma)\bbE_{\bs \sim p_0}[\nu^{tot}(\bs)] + \bbE_{(\bs, \ba) \sim \dmutot}\left[-\alpha f\left(w^{tot}_\nu(\bs, \ba)\right) + w^{tot}_\nu(\bs, \ba) A^{tot}_{\nu}(\bs, \ba)\right]
\end{align*}
% \[
% \cL(\nu^{tot}, w^{tot}) = (1 - \gamma)\bbE_{\bs \sim p_0}[\nu^{tot}(\bs)] + \bbE_{(\bs, \ba) \sim \dmutot}\left[-\alpha f\left(w^{tot}_\nu(\bs, \ba)\right) + w^{tot}_\nu(\bs, \ba) A^{tot}_{\nu}(\bs, \ba)\right]
% \]
where $A^{tot}_{\nu}$ is an ``advantage function'' defined based on $\nu^{tot}$ as:
\begin{equation}\label{eq:A-Q-v}
    A^{tot}_{\nu}(\bs, \ba) = q^{tot}(\bs, \ba) - \nu^{tot}(\bs)
\end{equation}
with $q^{tot}(\bs, \ba) = r(\cZ(\bs), \ba) + \gamma \bbE_{\bs' \sim P(\cdot | \bs, \ba)}[\nu^{tot}(\bs')]$. It is important to note that $\nu^{tot}(\bs)$ and $q^{tot}(\bs, \ba)$ can be interpreted as a value function and a Q-function, respectively, arising from the decomposition of the stationary distribution regularizer. We can now write the learning problem as follows:
\begin{align}
      \min\nolimits_{\nu^{tot}} \max\nolimits_{w^{tot} \geq 0} ~~\{\cL(\nu^{tot}, w^{tot})\}\label{prob:minimax-w-nu}
\end{align}
It can be observed that $\cL(\nu^{tot},w^{tot})$ is linear in $\nu^{tot}$ and concave in $w^{tot}$, which ensures well-behaved properties in both the $\nu^{tot}$- and $w^{tot}$-spaces. A key feature of the above minimax problem is that the inner maximization problem has a closed-form solution, which greatly simplifies the minimax problem, making it no longer adversarial. We formalize this result as follows:
\begin{proposition}\label{prop.1}
    The   minimax problem in Eq. \ref{prob:minimax-w-nu} is equivalent to 
$    \min_{\nu^{tot}} ~~ \big\{\widetilde{\cL}(\nu^{tot})\big\}$,  where
 \[
\widetilde{\cL}(\nu^{tot}) = (1 - \gamma)\bbE_{\bs \sim p_0}[\nu^{tot}(\bs)] + \bbE_{(\bs, \ba) \sim \dmutot}\bigg[\alpha  f^*\bigg(\frac{A^{tot}_{\nu}(\bs, \ba)}{\alpha}\bigg)\bigg]
 \]
Here, $f^*$ is  convex conjugate of $f$, i.e., $f^*(y) = \sup_{t\geq 0} \{ty - f(t)\}$. Moreover, if  $\nu^{tot}$ is parameterized by $\theta$, the first order derivative of $\widetilde{\cL}(\nu^{tot})$ w.r.t. $\theta$ is given as follows:
 \[
 \nabla_{\theta}  \widetilde{\cL}(\nu^{tot}) = (1 - \gamma)\bbE_{\bs \sim p_0}[\nabla_\theta \nu^{tot}(\bs)] + \bbE_{(\bs, \ba) \sim \dmutot}\left[\nabla_\theta A^{tot}_{\nu}(\bs, \ba)  w^{tot*}_{\nu}(\bs,\ba) \right]
 \]
where   $w^{tot*}_{\nu}(s,a)  = \max \{0, {f'}^{-1}(A^{tot}_\nu(\bs,\ba)/\alpha)\}$, with $f'^{-1}(\cdot)$ is the inverse function of the first-order derivative of $f$.
\end{proposition}

\subsection{Value Factorization}
Directly optimizing $ \min_{\nu^{tot}} ~~ \left\{\cL(\nu^{tot},w^{tot*}_{\nu})\right\}$ in multi-agent settings is generally impractical due to the large state and action spaces. Therefore, we follow the idea of value decomposition in the well-known CTDE framework in cooperative MARL to address this computational challenge. However, it is not straightforward to extend the DICE approach within this CTDE framework due to the complex objective of DICE, which involves the f-divergence between the learnt joint policy and the behavior policy in stationary distributions. It is thus crucial to carefully design the value decomposition in CTDE to ensure optimality consistency between the global and local policies.

Specifically, we adopt a factorization approach that decomposes the value function $\nu^{tot}(\bs)$ (or global Lagrange multipliers) into local values using mixing network architectures. Let $\bnu(\bs) = \{\nu_1(o_1), \ldots, \nu_n(o_n)\}$ represent a collection of local ``value functions'' and let $\bA_{\bnu}(\bs, \ba) = \{A_i(o_i,a_i),~i = 1,...,n\}$ represent a collection of local advantage functions. The local advantage functions are computed as $A_i(o_i,a_i) = q_i(o_i,a_i)  - \nu_i(o_i)$ for all $i \in \cN$, where  $\bq(\bs, \ba) = \{q_i(o_i,a_i), ~ i=1,...,n\}$ is a vector  of local Q-functions.  To facilitate centralized learning, we create a mixing network, $\cM_{\theta}$, where $\theta$ are the learnable weights, that aggregates the local values to form the global value and advantage functions as follows:
\[
\nu^{tot}(\bs,\ba) =  \cM_{\theta} [\bnu(\bs)], \quad A^{tot}_{\nu}(\bs,\ba) =  \cM_{\theta} [ \bq(\bs,\ba) - \bnu(\bs)],
\]
where each network takes the vectors $\bnu(\bs)$ or $\bA_{\bnu}(\bs,\ba)$ as inputs and outputs $\nu^{tot}$ and $A^{tot}_\nu$, respectively. Under this architecture, the learning objective becomes:
\[
\widetilde{\cL}(\bnu, \theta) =  (1 - \gamma)\bbE_{\bs \sim p_0}[\cM_{\theta} [\bnu(\bs)]] + \bbE_{(\bs, \ba) \sim \dmutot}\left[\alpha f^*\left(\frac{\cM_{\theta} [\bq(\bs,\ba) - \bnu(\bs)]}{\alpha}\right) \right],
\]
with the observation that $\bA_\nu(\bs,\ba)$ can be expressed as a linear function of $\bnu$. There are different ways to construct the mixing network $\cM_\theta$; prior works often employ a single linear combination (1-layer network) or a two-layer network with convex activations such as ReLU, ELU, or Maxout. In the following, we show a general result stating that the learning objective function is convex in $\bnu$, provided that the mixing network is constructed with non-negative weights and convex activations.
\begin{theorem}\label{th:convex}
If the mixing network $\cM_\theta[\cdot]$ is constructed with non-negative weights and convex activations, then $\widetilde{\cL}(\bnu,\theta)$ is convex in $\bnu$.
\end{theorem}

Theorem \ref{th:convex} shows that  $\widetilde{\cL}(\bnu,\theta)$ is convex in $\bnu$ when using \textit{any multi-layer feed-forward mixing networks with non-negative weights and convex activation functions}. This finding is highly general and non-trivial, given the nonlinearity and complexity of both the function (in terms of $\bnu$) and the mixing networks. Previous work has often focused on single-layer~\citep{wang2024offline_OMIGA} or two-layer mixing structures~\citep{rashid2020monotonic,bui2023inverse}, emphasizing that such two-layer networks can approximate any monotonic function arbitrarily closely as network width approaches infinity~\citep{dugas2009incorporating}. In our experiments, we test two configurations for the mixing network: a linear combination (or 1-layer) and a 2-layer feed-forward network. While 2-layer mixing structures have shown strong performance in online MARL~\citep{rashid2020monotonic,son2019qtran,wang2020qplex}, we observe in our offline settings that the linear combination approach provides more stable results.

 \subsection{Policy Extraction}

Let $\bnu^*$ be an optimal solution to the training problem with mixing networks, i.e., 
\begin{equation}\label{prob:main-L}
    \min\limits_{\bnu, \theta} \widetilde{\cL}(\bnu,\theta).
\end{equation}
We now need to extract a local and joint policy from this solution. 
% Recall that 
% \[
% w^{tot}_\nu(\bs, \ba) = \frac{\dtot(\bs,\ba)}{\dmutot(\bs,\ba)}.
% \]
Based on Prop.~\ref{prop.1}, given $\bnu^*$, we can compute this occupancy ratio as follows: :
\[
w^{tot*}(\bs, \ba) =  \max \left\{ 0, {f'}^{-1} \left( \frac{\cM_{\theta} [\bA_{\bnu^*}(\bs,\ba)]}{\alpha} \right) \right\}.
\]
The global policy can then be obtained as follows: 
$\bpi^*_{tot}(\ba|\bs) = \frac{w^{tot*}(\bs, \ba) \cdot \dmutot(\bs,\ba)}{\sum_{\ba' \in \mathcal{A}} w^{tot*}(\bs, \ba') \cdot \dmutot(\bs,\ba')}.
$
This computation, however, is not practical since $\dmutot$ is generally not available and might not be accurately estimated in the offline setting. A more practical way to estimate the global policy, $\bpi^*_{tot}$, as the result of solving the following weighted behavioral cloning (BC): 
\begin{align}
\max_{\bpi_{tot} \in \Pi_{tot}} \bbE_{(\bs,\ba) \sim \rho^{\bpi^*_{tot}}} [\log \bpi_{tot}(\ba|\bs)] 
%&= \max_{\bpi_{tot}} \sum_{\ba,\bs} \rho^{\bpi^*_{tot}}(\bs,\ba) \log \bpi_{tot}(\bs,\ba) \nonumber \\
% &= \max_{\bpi_{tot}} \sum_{\ba,\bs} \dmutot(\bs,\ba) w^{tot*}(\bs,\ba) \log \bpi_{tot}(\bs,\ba) \nonumber \\
= \max_{\bpi_{tot}\in \Pi_{tot}} \bbE_{(\bs,\ba) \sim \dmutot} [ w^{tot*}(\bs,\ba) \log \bpi_{tot}(\ba|\bs)]\label{prob:pi-tot}
\end{align}
where $\Pi_{tot}$ represents the feasible set of global policies. Here we assume that $\Pi_{tot}$ contains decomposable global policies, i.e., $\Pi_{tot} = \{\bpi_{tot} \mid \exists \pi_i, ~\forall i \in \mathcal{N} \text{ such that } \bpi_{tot}(\ba|\bs) = \prod_{i\in \mathcal{N}} \pi_i(a_i|o_i)\}$. In other words, $\Pi_{tot}$ consists of global policies that can be expressed as a product of local policies. This decomposability is highly useful for decentralized learning and has been widely adopted in MARL \citep{wang2024offline_OMIGA,bui2023inverse,zhang2021fop}.

While the above  weighted BC  appears practical, as $(\bs, \ba)$ can be sampled from the offline dataset generated by $\dtot$, and since $w^{tot*}(\bs, \ba)$ is available from solving \ref{prob:main-L}, it does not directly yield local policies, which are essential for decentralized execution. To address this, we propose solving the following weighted BC for each local agent \( i \in \mathcal{N} \):
\[
\max\nolimits_{\pi_i} \mathbb{E}_{(\bs, \ba) \sim \mathcal{D}} \left[ w^{tot*}(\bs, \ba) \log \pi_i(a_i|o_i) \right]. \label{prob:pi-local}
\]
This local WBC approach has several attractive properties. First, $w^{tot*}(\bs, \ba)$ explicitly appears in the local policy optimization and is computed from global observations and actions. This enables local policies to be optimized with global information, ensuring consistency with the credit assignment in the multi-agent system. Furthermore, as shown in Proposition~\ref{prop.consistency} below, the optimization of local policies through the lobcal WBC is highly consistent with the global  weighted BC in \ref{prob:pi-tot}.
\begin{proposition}\label{prop.consistency}
    Let \( \pi^* \) be the optimal solution to \ref{prob:pi-local}. Then \( \boldsymbol{\pi}^*_{tot}(\ba|\bs) = \prod_{i \in \mathcal{N}} \pi^*_i(a_i|o_i) \) is also optimal for the global weighted BC in \ref{prob:pi-tot}.
\end{proposition}
Here we note that consistency between global and local policies is a critical aspect of centralized training with CTDE. Previous MARL approaches typically achieve this by factorizing the Q- or V-functions into local functions, and training local policies based on these local ones \citep{rashid2020monotonic,wang2020qplex,bui2023inverse}. However, in our case, there are key differences that prevent us from employing such local values to derive local policies. Specifically, we factorize the Lagrange multipliers \( \nu^{tot} \) to train the stationary distribution ratio \( w^{tot} \). While local \( w \) values can be extracted from the local \( \nu_i \), these local \( w \) values do not represent a local stationary distribution ratio and therefore cannot be used to recover local policies.

\section{Practical Algorithm}
\begin{algorithm}[htb]
\caption{\textbf{ ComaDICE}: Offline \textbf{Co}operative \textbf{MA}RL with Stationary \textbf{DI}stribution \textbf{C}orrection \textbf{E}stimation}
\label{algo:ComaDICE}
    \centering
 \begin{algorithmic}[1]
 \STATE \textbf{Input:} Parameters $\theta, \psi_q, \psi_\nu,\eta_i$ and the corresponding learning rates $\lambda_\theta,\lambda_{\psi_q},\lambda_{\psi_\nu},\lambda_\eta$, respectively. Offline data $\cD$.
 \STATE \textbf{Output:} Local optimized polices $\pi_i$.
 \STATE \cmt{\# Training the occupancy ratio $w^{tot*}$}  
\FOR{\textit{a certain number of training steps}}
     % \STATE  
       \STATE $\psi_q =  \psi_q  - \lambda_{\psi_q} \nabla_{\psi_q}\cL(\psi_q)$ \quad\cmt{\# Update Q-function towards the MSE  in \ref{eq:Q-MSE}} 
        % \STATE  
        \STATE $\theta =  \theta   - \lambda_{\theta} \nabla_{\theta}\widetilde{\cL}(\psi_\nu,\theta)$ \quad\cmt{\# Update $\theta$ to minimize the loss in \ref{eq:main-loss}}
        \STATE $\psi_\nu =  \psi_\nu   - \lambda_{\psi_\nu} \nabla_{\psi_\nu}\widetilde{\cL}(\psi_\nu,\theta)$\quad\cmt{\# Update $\psi_\nu$ to minimize the loss in \ref{eq:main-loss}}
   \ENDFOR
    \STATE \cmt{\# Training local policy}  
\FOR{\textit{a certain number of training steps}}
     % \STATE  
       \STATE $\eta_i =  \eta_i + \lambda_{\eta} \nabla_{\eta_i}\cL_{\pi}(\eta_i)$ \quad\cmt{\# Update the local policy by optimizing \ref{eq:WBC}}
   \ENDFOR
   \STATE Return $\pi_i(a_i|o_i;\eta_i)$, $i=1,...,n$
 \end{algorithmic}
\end{algorithm}

Let $\cD$ represent the offline dataset, consisting of sequences of local observations and actions gathered from a global behavior policy $\pmb{\pi}_{tot}$. To train the value function $\bnu$, we construct a value network $\nu_i(o_i;\psi_{\nu})$ for each local agent $i$, along with a network for each local Q-function $q_i(o_i,a_i;\psi_q)$, where $\psi_\nu$  and $\psi_q$ are learnable parameters for the local value and Q-functions. Each local advantage function is then calculated as follows: 
The global value function and advantage function are subsequently aggregated using two mixing networks with a shared set of learnable parameters $\theta$:
\[
\nu^{tot}(\bs) = \cM^\bs_{\theta} [\bnu(\bs;\psi_\nu)], \quad A^{tot}_{\nu}(\bs,\ba) = \cM^\bs_{\theta} [\bq(\bs,\ba;\psi_q) - \bnu(\bs; \psi_\nu)],
\]
where $\cM^\bs_\theta[\cdot]$ represents a linear combination of its inputs with non-negative weights, such that $\cM^\bs_\theta[\bnu(\bs;\psi_\nu)] = \bnu(\bs;\psi_\nu)^\top W^\bs_\theta + b^\bs_\theta$, where $W^\bs_\theta$ and $b^\bs_\theta$ are weights of the mixing network.\footnote{In our experiments, we use a single-layer mixing network due to its superior performance compared to a two-layer structure, though our approach is general and can handle any multi-layer feed-forward mixing network.} It is important to note that $W^\bs_\theta$ and $b^\bs_\theta$ are generated by hyper-networks that take the global state $\bs$ and the learnable parameters $\theta$ as inputs. In this context, we employ the same mixing network $\cM^\bs_{\theta}$ to combine the local values and advantages. However, our framework is flexible enough to allow the use of two different mixing networks for $\nu^{tot}$ and $A_\nu^{tot}$.

In our setting, the relationship between the global Q-function, value, and advantage functions is described in Eq. \ref{eq:A-Q-v}. Specifically, we have:
 $ A^{tot}_{\nu}(\bs, \ba) =  r(\cZ(\bs), \ba) + \gamma \bbE_{\bs' \sim P(\cdot | \bs, \ba)}[\nu^{tot}(\bs')] - \nu^{tot}(\bs).$
To  capture this relationship, we train the Q-function by optimizing the following MSE loss:
\[
 \min\nolimits_{\bq}  \sum\nolimits_{(\bs,\ba,\bs')\sim \cD} \left( A^{tot}_{\nu}(\bs, \ba)  -  r(\cZ(\bs), \ba) + \gamma \nu^{tot}(\bs')- \nu^{tot}(\bs) \right)^2.
\]
This is equivalent to:
\begin{align}
     \min\nolimits_{\psi_q} \cL_q(\psi_q) =   \sum\nolimits_{(\bs,\ba,\bs')\sim \cD} &\Big( \cM^\bs_{\theta} [\bq(\bs,\ba;\psi_q) - \bnu(\bs; \psi_\nu)]   \nonumber\\
     &-  r(\cZ(\bs), \ba) + \gamma \cM^{\bs'}_{\theta} [\bnu(\bs';\psi_\nu)] - \cM^\bs_{\theta} [\bnu(\bs;\psi_\nu)] \Big)^2\label{eq:Q-MSE}
\end{align}
%where we use the same mixing network $\cM^\bs_{\theta}$ to aggregate the local values and advantages.

For the primary loss function used to train the value function, we leverage transitions from the offline dataset to approximate the objective $\widetilde{\cL}$, resulting in the following loss function for offline training:
\begin{equation}\label{eq:main-loss}
    \widetilde{\cL}(\psi_\nu, \theta) =  (1 - \gamma)\bbE_{\bs_0 \sim \cD}[\cM^{\bs_0}_{\theta} [\bnu(\bs_0;\psi_\nu)]] + \bbE_{(\bs, \ba) \sim \cD}\left[\alpha f^*\!\!\left(\frac{\cM^{\bs}_{\theta} [\bq(\bs,\ba;\psi_q) - \bnu(\bs;\psi_\nu)]}{\alpha}\right) \right]
\end{equation}
As mentioned, after obtaining $(\bnu^*, \theta^*)$ by solving $\min_{\psi_\nu, \theta}\widetilde{\cL}(\psi_\nu, \theta)$, we compute the occupancy ratio:
$w^{tot*}_{\nu}(\bs,\ba)  = \max \left\{ 0, {f'}^{-1} \left( \frac{\cM^{\bs}_{\theta^*} [\bnu^*(\bs)] - \cM^\bs_{\theta^*} [\bq(\bs,\ba;\psi_q)]}{\alpha} \right) \right\}.$
To train the local policy $\pi_i(a_i|o_i)$, we represent it using a policy network $\pi_i(a_i|o_i; \eta_i)$, where $\eta_i$ are the learnable parameters. The training process involves optimizing the following weighted behavioral cloning (BC) objective:
\begin{equation}\label{eq:WBC}
    \max\nolimits_{\eta_i} ~~ \cL_\pi(\eta_i) = \sum\nolimits_{(\bs,\ba) \sim \cD} w^{tot*}_{\nu}(\bs,\ba) \log( \pi_i(a_i|o_i;\eta_i)).
\end{equation}
Our ComaDICE algorithm consists of two primary steps. The first step involves estimating the occupancy ratio \( w^{tot*} \) from the offline dataset. The second step focuses on training the local policy by solving the weighted BC problem using \( w^{tot*} \). In the first step, we simultaneously update the Q-functions \( \psi_q \), the mixing network parameters \( \theta \), and the value function \( \psi_\nu \), aiming to minimize the mean squared error (MSE) in Eq. \ref{eq:Q-MSE} while optimizing the main loss function in Eq. \ref{eq:main-loss}.

\section{Experiments}

\subsection{Environments}
We utilize three standard MARL environments: SMACv1~\citep{samvelyan2019starcraft}, SMACv2~\citep{ellis2022smacv2}, and Multi-Agent MuJoCo (MaMujoco)~\citep{de2020deep}, each offering unique challenges and configurations for evaluating cooperative MARL algorithms.
% These datasets have been widely used to assess the performance of SOTA MARL algorithm \ct{}
%form a robust foundation for assessing the performance and scalability of MARL algorithms across diverse scenarios.
% \smallskip

\noindent\textbf{SMACv1.} SMACv1 is based on Blizzard's StarCraft II. It uses the StarCraft II API and DeepMind's PySC2 to enable agent interactions with the game. SMACv1 focuses on decentralized micromanagement scenarios where each unit is controlled by an RL agent. Tasks like \emph{2c\_vs\_64zg} and \emph{5m\_vs\_6m} are labeled hard, while \emph{6h\_vs\_8z} and \emph{corridor} are super hard. The offline dataset, provided by \citet{meng2023offline}, was generated using MAPPO-trained agents \citep{yu2022surprising}.

% \smallskip
\noindent\textbf{SMACv2.} In comparison to SMACv1, SMACv2 introduces increased randomness and diversity by randomizing start positions, unit types, and modifying sight and attack ranges. This version includes tasks such as \emph{protoss}, \emph{terran}, and \emph{zerg}, with instances ranging from \emph{5\_vs\_5} to \emph{20\_vs\_23}, increasing in difficulty. Our offline dataset for SMACv2 was generated by running MAPPO for 10 million training steps and collecting 1,000 trajectories, ensuring medium quality but comprehensive coverage of the learning process. To the best of our knowledge, we are the first to explore SMACv2 in offline MARL, whereas most prior work has used this environment in online settings.

% \smallskip
\noindent\textbf{MaMujoco.} MaMujoco serves as a benchmark for continuous cooperative multi-agent robotic control. Derived from the single-agent MuJoCo control suite in OpenAI Gym \citep{brockman2016openai}, it presents scenarios where multiple agents within a single robot must collaborate to achieve tasks. The tasks include \emph{Hopper-v2}, \emph{Ant-v2}, and \emph{HalfCheetah-v2}, with instances labeled as \emph{expert}, \emph{medium}, \emph{medium-replay}, and \emph{medium-expert}. The offline dataset was created by  \citep{wang2024offline_OMIGA} using the HAPPO method \citep{wang2022trust_HAPPO}.

\subsection{Baselines}
We consider the following baselines, which represent either standard or state-of-the-art (SOTA) methods for offline MARL: (i) \textbf{BC} (Behavioral Cloning); (ii) \textbf{BCQ} (Batch-Constrained Q-learning) \citep{fujimoto2019off} – an offline RL algorithm that constrains the policy to actions similar to those in the dataset to reduce distributional shift, adapted for offline MARL settings; (iii) \textbf{CQL} (Conservative Q-Learning) \citep{kumar2020conservative} – a method that stabilizes offline Q-learning by penalizing out-of-distribution actions, ensuring conservative value estimates; (iv) \textbf{ICQ} (Implicit Constraint Q-learning) \citep{yang2021believe} – an approach using importance sampling to manage out-of-distribution actions in multi-agent settings; (v) \textbf{OMAR} (Offline MARL with Actor Rectification) \citep{pan2022plan} – a method combining CQL with optimization techniques to ensure the global validity of local regularizations, promoting cooperative behavior; (vi) \textbf{OMIGA} (Offline MARL with Implicit Global-to-Local Value Regularization) \citep{wang2024offline_OMIGA} – a SOTA method that transforms global regularizations into implicit local ones, optimizing local policies with global insights.

\begin{table}[t!]\small
\small
\centering
\resizebox{\textwidth}{!}{
\begin{tabular}{cc|ccccccc}
\toprule
\multicolumn{2}{c|}{\multirow{2}{*}{\textbf{Instances}}} & \multirow{2}{*}{\textbf{BC}} & \multirow{2}{*}{\textbf{BCQ}} & \multirow{2}{*}{\textbf{CQL}} & \multirow{2}{*}{\textbf{ICQ}} & \multirow{2}{*}{\textbf{OMAR}} & \multirow{2}{*}{\textbf{OMIGA}} & \textbf{ComaDICE} \\
\multicolumn{2}{c|}{} &  &  &  &  &  &  & (ours) \\
\midrule
\multirow{5}{*}{Protoss} & 5\_vs\_5 & 36.9±8.7 & 16.2±2.3 & 10.0±4.1 & 36.9±9.1 & 21.2±4.1 & 33.1±5.4 & \red{46.2±6.1} \\
 & 10\_vs\_10 & 36.2±10.6 & 9.4±5.6 & 26.2±7.6 & 28.1±6.6 & 13.8±7.0 & 40.0±10.7 & \red{50.6±8.7} \\
 & 10\_vs\_11 & 19.4±4.6 & 10.0±4.1 & 10.6±5.4 & 12.5±4.4 & 12.5±3.4 & 16.2±6.1 & \red{20.0±4.2} \\
 & 20\_vs\_20 & 37.5±4.4 & 6.2±2.0 & 11.9±4.1 & 32.5±8.1 & 23.8±2.5 & 36.2±5.1 & \red{47.5±7.8} \\
 & 20\_vs\_23 & 13.8±1.5 & 1.2±1.5 & 0.0±0.0 & 12.5±5.6 & 11.2±7.8 & 12.5±8.1 & \red{13.8±5.8} \\
\midrule
\multirow{5}{*}{Terran} & 5\_vs\_5 & 30.0±4.2 & 12.5±6.2 & 9.4±7.9 & 23.1±5.8 & 14.4±4.7 & 28.1±4.4 & \red{30.6±8.2} \\
 & 10\_vs\_10 & 29.4±5.8 & 6.9±6.1 & 9.4±5.6 & 16.9±5.8 & 15.0±4.6 & 29.4±3.2 & \red{32.5±5.8} \\
 & 10\_vs\_11 & 16.2±3.6 & 3.8±4.6 & 7.5±6.4 & 5.0±4.2 & 9.4±5.6 & 12.5±5.2 & \red{19.4±5.4} \\
 & 20\_vs\_20 & 26.2±10.4 & 5.0±3.2 & 10.6±4.2 & 15.6±3.4 & 7.5±7.3 & 21.9±4.4 & \red{29.4±3.8} \\
 & 20\_vs\_23 & 4.4±4.2 & 0.0±0.0 & 0.0±0.0 & 7.5±6.1 & 5.0±4.2 & 4.4±2.5 & \red{9.4±5.2} \\
\midrule
\multirow{5}{*}{Zerg} & 5\_vs\_5 & 26.9±10.0 & 14.4±4.2 & 14.4±5.8 & 18.8±7.1 & 13.8±6.1 & 21.9±5.9 & \red{31.2±7.7} \\
 & 10\_vs\_10 & 25.0±2.8 & 5.6±4.6 & 5.6±4.6 & 15.6±7.4 & 19.4±2.3 & 23.8±6.4 & \red{33.8±11.8} \\
 & 10\_vs\_11 & 13.8±4.7 & 9.4±5.2 & 6.2±4.4 & 10.6±6.7 & 10.6±3.8 & 13.8±6.7 & \red{19.4±3.6} \\
 & 20\_vs\_20 & 8.1±1.5 & 2.5±1.2 & 1.2±1.5 & 10.0±7.8 & \red{12.5±4.4} & 10.0±2.3 & 9.4±6.2 \\
 & 20\_vs\_23 & 7.5±3.2 & 0.6±1.3 & 1.2±1.5 & 7.5±3.2 & 3.8±2.3 & 4.4±4.2 & \red{11.2±4.2} \\
\bottomrule
\end{tabular}
}
\caption{Comparison of winrates for ComaDICE and baselines across SMACv2 tasks.
%, highlighting ComaDICE's superior performance and stability in multi-agent coordination tasks.
}
\label{SMAC:main:winrate}
\end{table} 
We used experimental results contributed by the authors of OMIGA \citep{wang2024offline_OMIGA} as our baselines. They provided both the results and source code for all the baseline methods. This source code was also employed to run these baselines for the SMACv2 environment. All hyperparameters were kept at their default settings, and each experiment was conducted with\textit{ five different random seeds } to ensure robustness and reproducibility of the results.

\subsection{Main Comparison}
We now present a comprehensive evaluation of our proposed algorithm, ComaDICE, against several baseline methods in offline MARL. The baselines selected for comparison include both standard and SOTA approaches, providing a robust benchmark to assess the effectiveness of ComaDICE.

Our evaluation focuses on two primary metrics: returns and winrates. Returns are the average rewards accumulated by the agents across multiple trials, providing a measure of policy effectiveness. Winrates, applicable in competitive environments such as SMACv1 and SMACv2, indicate the success rate of agents against opponents, reflecting the algorithm's robustness in adversarial settings.

The experimental results, summarized in Tables \ref{SMAC:main:winrate} and \ref{MaMujoco:main:return}, demonstrate that ComaDICE consistently achieves superior performance compared to baseline methods across a range of scenarios. Notably, ComaDICE excels in complex tasks, highlighting its ability to effectively manage distributional shifts in challenging environments.

Figures \ref{SMAC:main:curves:winrates} illustrates the learning curves for each algorithm, showing that ComaDICE not only outperforms other methods in terms of mean returns but also exhibits lower variance, indicating more stable and reliable performance. These findings underscore the robustness and adaptability of ComaDICE, setting a new benchmark for offline MARL.

\begin{table}[t!]\small
\small
\centering
\resizebox{\textwidth}{!}{
\begin{tabular}{cc|ccccc}
\toprule
\multicolumn{2}{c|}{\multirow{2}{*}{\textbf{Instances}}} & \multirow{2}{*}{\textbf{BCQ}} & \multirow{2}{*}{\textbf{CQL}} & \multirow{2}{*}{\textbf{ICQ}} & \multirow{2}{*}{\textbf{OMIGA}} & \textbf{ComaDICE} \\
\multicolumn{2}{c|}{} &  &  &  &  & (ours) \\
\midrule
\multirow{4}{*}{Hopper} & expert & 77.9±58.0 & 159.1±313.8 & 754.7±806.3 & 859.6±709.5 & \red{2827.7±62.9} \\
 & medium & 44.6±20.6 & 401.3±199.9 & 501.8±14.0 & \red{1189.3±544.3} & 822.6±66.2 \\
 & m-replay & 26.5±24.0 & 31.4±15.2 & 195.4±103.6 & 774.2±494.3 & \red{906.3±242.1} \\
 & m-expert & 54.3±23.7 & 64.8±123.3 & 355.4±373.9 & 709.0±595.7 & \red{1362.4±522.9} \\
\midrule
\multirow{4}{*}{Ant} & expert & 1317.7±286.3 & 1042.4±2021.6 & 2050.0±11.9 & 2055.5±1.6 & \red{2056.9±5.9} \\
 & medium & 1059.6±91.2 & 533.9±1766.4 & 1412.4±10.9 & 1418.4±5.4 & \red{1425.0±2.9} \\
 & m-replay & 950.8±48.8 & 234.6±1618.3 & 1016.7±53.5 & 1105.1±88.9 & \red{1122.9±61.0} \\
 & m-expert & 1020.9±242.7 & 800.2±1621.5 & 1590.2±85.6 & 1720.3±110.6 & \red{1813.9±68.4} \\
\midrule
\multirow{4}{*}{\shortstack{Half\\Cheetah}} & expert & 2992.7±629.7 & 1189.5±1034.5 & 2955.9±459.2 & 3383.6±552.7 & \red{4082.9±45.7} \\
 & medium & 2590.5±1110.4 & 1011.3±1016.9 & 2549.3±96.3 & \red{3608.1±237.4} & 2664.7±54.2 \\
 & m-replay & -333.6±152.1 & 1998.7±693.9 & 1922.4±612.9 & 2504.7±83.5 & \red{2855.0±242.2} \\
 & m-expert & 3543.7±780.9 & 1194.2±1081.0 & 2834.0±420.3 & 2948.5±518.9 & \red{3889.7±81.6} \\
\bottomrule
\end{tabular}
}
\caption{Average returns for ComaDICE and baselines on MaMuJoCo benchmarks.
%, demonstrating ComaDICE's enhanced effectiveness and consistency in complex environments.
}
\label{MaMujoco:main:return}
\end{table}

\begin{figure}[t!]\small
\centering
\showlegend[0.8]{10}{6.8}
\\
\rotatebox{90}{Protoss} \hspace{4pt}
\showinstance[0.20]{0}{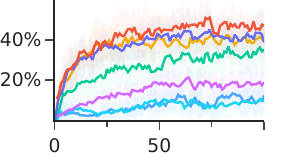}{}
\showinstance[0.17]{20}{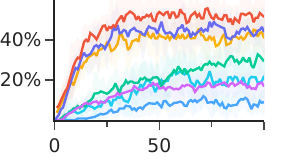}{}
\showinstance[0.17]{20}{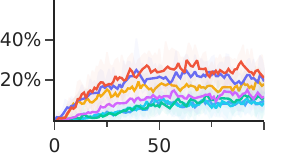}{}
\showinstance[0.17]{20}{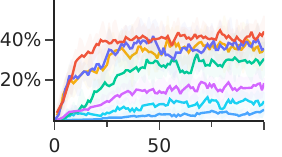}{}
\showinstance[0.17]{20}{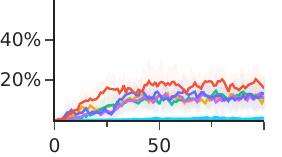}{}
\\
\rotatebox{90}{Terran} \hspace{4pt}
\showinstance[0.20]{0}{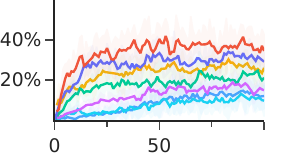}{}
\showinstance[0.17]{20}{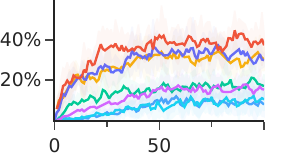}{}
\showinstance[0.17]{20}{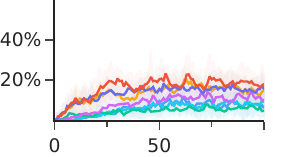}{}
\showinstance[0.17]{20}{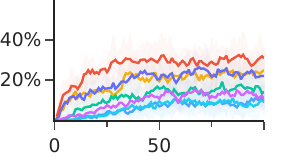}{}
\showinstance[0.17]{20}{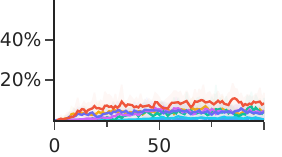}{}
\\
\rotatebox{90}{Zerg} \hspace{4pt}
\showinstance[0.20]{0}{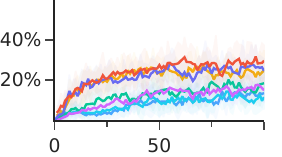}{\emph{5\_vs\_5}}
\showinstance[0.17]{20}{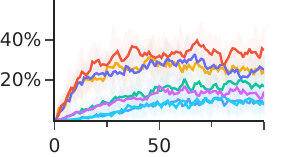}{\emph{10\_vs\_10}}
\showinstance[0.17]{20}{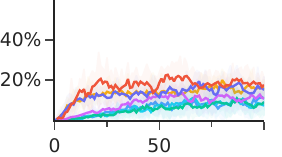}{\emph{10\_vs\_11}}
\showinstance[0.17]{20}{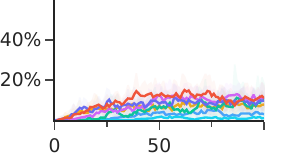}{\emph{20\_vs\_20}}
\showinstance[0.17]{20}{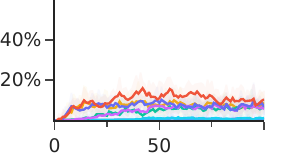}{\emph{20\_vs\_23}}
\\
\caption{Evaluation curves of ComaDICE and baselines over time on SMACv2 tasks.}
\label{SMAC:main:curves:winrates}
\end{figure}

\subsection{Ablation Study -  Impact of the Regularization Parameter Alpha}
We investigate how varying the regularization parameter alpha ($\alpha$) affects the performance of our ComaDICE algorithm. The parameter $\alpha$ is crucial for balancing the trade-off between maximizing rewards and penalizing deviations from the offline dataset's distribution. We conducted experiments with $\alpha$ values ranging from $\{0.01, 0.1, 1, 10, 100\}$, evaluating performance using winrates in the SMACv2 environment and returns in the MaMujoco environment. These results, illustrated in Figure \ref{boxes:main:returns}, highlight the sensitivity of ComaDICE to different $\alpha$ values. In particular, we observe that ComaDICE achieves optimal performance when $\alpha$ is around 10, suggesting that the stationary distribution regularizer plays a essential role in the success of our algorithm. 

In our appendix, we provide additional ablation studies to analyze the performance of our algorithm using different forms of f-divergence functions, as well as comparisons between 1-layer and 2-layer mixing network structures. The appendix also includes proofs of the theoretical claims made in the main paper, details of our experimental settings, and other experimental information.

%providing insights into optimal parameter settings for robust performance across diverse scenarios.

\begin{figure}[t!]\small
\small
\centering
% \showinstance[0.17]{0}{protoss_box_returns}{Protoss}
% \showinstance[0.14]{12}{terran_box_returns}{Terran}
% \showinstance[0.14]{12}{zerg_box_returns}{Zerg}
\showinstance[0.17]{0}{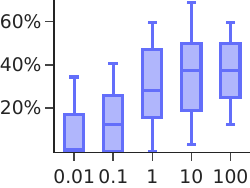}{protoss}
\showinstance[0.14]{20}{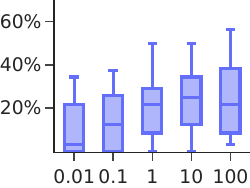}{terran}
\showinstance[0.14]{20}{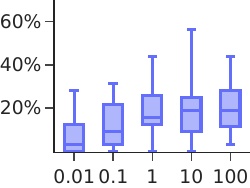}{zerg}
\showinstance[0.17]{0}{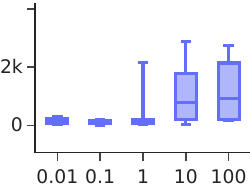}{Hopper}
\showinstance[0.14]{12}{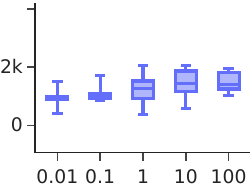}{Ant}
\showinstance[0.14]{12}{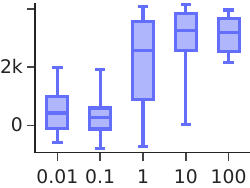}{HalfCheetah}
\caption{Impact of regularization parameter $\alpha$ on performance in different environments.}
\label{boxes:main:returns}
\end{figure}

\section{Conclusion, Future Work and Broader Impacts}

\textbf{{Conclusion.}} In this paper, we propose ComaDICE, a principled framework for offline MARL. Our algorithm incorporates a stationary distribution shift regularizer into the standard MARL objective to address the conventional distribution shift issue in offline RL. To facilitate training within a CTDE framework, we decompose both the global value and advantage functions using a mixing network. We demonstrate that, under our mixing architecture, the main objective function is concave in the value function, which is crucial for ensuring stable and efficient training. The results of this training are then utilized to derive local policies through a weighted BC approach, ensuring consistency between global and local policy optimization. Extensive experiments on SOTA benchmark tasks, including SMACv2, show that ComaDICE outperforms other baseline methods. 
% \smallskip

\textbf{{Limitations and Future Work:} }There are some limitations that are not addressed within the scope of this paper. For instance, we focus solely on cooperative learning, leaving open the question of how the approach would perform in cooperative-competitive settings. Extending ComaDICE to such scenarios would require considerable effort and is an interesting direction for future research. Additionally, in our training objective, the DICE term is designed to reduce the divergence between the learning policy and the behavior policy. As a result, the performance of the algorithm is heavily dependent on the quality of the behavior policy. Although this reliance may be unavoidable, future research should focus on mitigating the influence of the behavior policy on training outcomes. Furthermore, our algorithm, like other baselines, still requires a large amount of data to achieve desirable learning outcomes. Improving sample efficiency would be another valuable area for future research. 

\textbf{{Broader Impacts:}} The development of an offline MARL algorithm using a stationary distribution shift regularizer could lead to improved performance in tasks where real-time interaction is costly, such as robotics, autonomous driving, and healthcare. It could also promote safer exploration and wider adoption of offline learning in high-stakes environments. On the negative side, since the algorithm relies heavily on the behavior policy, if the behavior policy is flawed or biased, the performance of the learnt policy could also suffer. This could reinforce preexisting biases or suboptimal behaviors in real-world applications. Moreover, like any AI technology, there is a risk of the algorithm being applied in unintended or harmful ways, such as in surveillance or military applications, where multi-agent systems could be used to manipulate environments or people without adequate oversight.

\section*{Ethical Statement}

Our work introduces ComaDICE, a framework for offline MARL, aimed at improving training stability and policy optimization in complex multi-agent environments. While this research has significant potential for positive applications, particularly in domains such as autonomous systems, resource management, and multi-agent simulations, it is crucial to address the ethical implications and risks associated with this technology.

The deployment of reinforcement learning systems in real-world, multi-agent settings raises concerns about unintended behaviors, especially in safety-critical domains. If the policies learned by ComaDICE are applied without proper testing and validation, they may lead to undesirable or harmful outcomes, especially in areas such as autonomous driving, healthcare, or robotics. Additionally, bias in the training data or simulation environments could result in suboptimal policies that unfairly impact certain agents or populations, potentially leading to ethical concerns regarding fairness and transparency.

To mitigate these risks, we emphasize the need for extensive testing and validation of policies generated using ComaDICE, particularly in real-world environments where the consequences of errors could be severe. It is also essential to ensure that the datasets and simulations used in training are representative, unbiased, and carefully curated. We encourage practitioners to use human oversight and collaborate with domain experts to ensure that ComaDICE is applied responsibly, particularly in high-stakes settings.

\section*{Reproducibility Statement}
In order to facilitate reproducibility, we have submitted the source code for ComaDICE, along with the datasets utilized to produce the experimental results presented in this paper (all these will be made publicly available if the paper gets accepted). 
Additionally, in the appendix, we provide details of our algorithm, including key implementation steps and details needed to replicate the results. The hyper-parameter settings for all experiments are also included to ensure that others can reproduce the findings under the same experimental conditions. We invite the research community to explore and apply the ComaDICE framework in various environments to further validate and expand upon the results reported in this work.

\bibliography{refs}

\begin{thebibliography}{63}
\providecommand{\natexlab}[1]{#1}
\providecommand{\url}[1]{\texttt{#1}}
\expandafter\ifx\csname urlstyle\endcsname\relax
  \providecommand{\doi}[1]{doi: #1}\else
  \providecommand{\doi}{doi: \begingroup \urlstyle{rm}\Url}\fi

\bibitem[Agarwal et~al.(2020)Agarwal, Schuurmans, and Norouzi]{agarwal2020optimistic}
Rishabh Agarwal, Dale Schuurmans, and Mohammad Norouzi.
\newblock An optimistic perspective on offline reinforcement learning.
\newblock In \emph{International conference on machine learning}, pp.\  104--114. PMLR, 2020.

\bibitem[An et~al.(2021)An, Moon, Kim, and Song]{an2021uncertainty}
Gaon An, Seungyong Moon, Jang-Hyun Kim, and Hyun~Oh Song.
\newblock Uncertainty-based offline reinforcement learning with diversified q-ensemble.
\newblock \emph{Advances in neural information processing systems}, 34:\penalty0 7436--7447, 2021.

\bibitem[Bai et~al.(2022)Bai, Wang, Yang, Deng, Garg, Liu, and Wang]{bai2022pessimistic}
Chenjia Bai, Lingxiao Wang, Zhuoran Yang, Zhihong Deng, Animesh Garg, Peng Liu, and Zhaoran Wang.
\newblock Pessimistic bootstrapping for uncertainty-driven offline reinforcement learning.
\newblock \emph{arXiv preprint arXiv:2202.11566}, 2022.

\bibitem[Brandfonbrener et~al.(2021)Brandfonbrener, Whitney, Ranganath, and Bruna]{brandfonbrener2021offline}
David Brandfonbrener, Will Whitney, Rajesh Ranganath, and Joan Bruna.
\newblock Offline rl without off-policy evaluation.
\newblock \emph{Advances in neural information processing systems}, 34:\penalty0 4933--4946, 2021.

\bibitem[Brockman et~al.(2016)Brockman, Cheung, Pettersson, Schneider, Schulman, Tang, and Zaremba]{brockman2016openai}
Greg Brockman, Vicki Cheung, Ludwig Pettersson, Jonas Schneider, John Schulman, Jie Tang, and Wojciech Zaremba.
\newblock Openai gym, 2016.
\newblock URL \url{http://arxiv.org/abs/1606.01540}.

\bibitem[Bui et~al.(2024)Bui, Mai, and Nguyen]{bui2023inverse}
The~Viet Bui, Tien Mai, and Thanh~Hong Nguyen.
\newblock Inverse factorized q-learning for cooperative multi-agent imitation learning.
\newblock \emph{Advances in Neural Information Processing Systems}, 38, 2024.

\bibitem[Chen et~al.(2020)Chen, Zhou, Wang, Wang, Wu, and Ross]{chen2020bail}
Xinyue Chen, Zijian Zhou, Zheng Wang, Che Wang, Yanqiu Wu, and Keith Ross.
\newblock Bail: Best-action imitation learning for batch deep reinforcement learning.
\newblock \emph{Advances in Neural Information Processing Systems}, 33:\penalty0 18353--18363, 2020.

\bibitem[Cheng et~al.(2024)Cheng, Zhan, Zhang, Lin, Wang, Jiang, et~al.]{cheng2024look}
Peng Cheng, Xianyuan Zhan, Wenjia Zhang, Youfang Lin, Han Wang, Li~Jiang, et~al.
\newblock Look beneath the surface: Exploiting fundamental symmetry for sample-efficient offline rl.
\newblock \emph{Advances in Neural Information Processing Systems}, 36, 2024.

\bibitem[de~Witt et~al.(2020)de~Witt, Peng, Kamienny, Torr, B{\"o}hmer, and Whiteson]{de2020deep}
Christian~Schroeder de~Witt, Bei Peng, Pierre-Alexandre Kamienny, Philip Torr, Wendelin B{\"o}hmer, and Shimon Whiteson.
\newblock Deep multi-agent reinforcement learning for decentralized continuous cooperative control.
\newblock \emph{arXiv preprint arXiv:2003.06709}, 19, 2020.

\bibitem[Dugas et~al.(2009)Dugas, Bengio, B{\'e}lisle, Nadeau, and Garcia]{dugas2009incorporating}
Charles Dugas, Yoshua Bengio, Fran{\c{c}}ois B{\'e}lisle, Claude Nadeau, and Ren{\'e} Garcia.
\newblock Incorporating functional knowledge in neural networks.
\newblock \emph{Journal of Machine Learning Research}, 10\penalty0 (6), 2009.

\bibitem[Ellis et~al.(2022)Ellis, Moalla, Samvelyan, Sun, Mahajan, Foerster, and Whiteson]{ellis2022smacv2}
Benjamin Ellis, Skander Moalla, Mikayel Samvelyan, Mingfei Sun, Anuj Mahajan, Jakob~N Foerster, and Shimon Whiteson.
\newblock Smacv2: An improved benchmark for cooperative multi-agent reinforcement learning.
\newblock \emph{arXiv preprint arXiv:2212.07489}, 2022.

\bibitem[Fujimoto \& Gu(2021)Fujimoto and Gu]{fujimoto2021minimalist}
Scott Fujimoto and Shixiang~Shane Gu.
\newblock A minimalist approach to offline reinforcement learning.
\newblock \emph{Advances in neural information processing systems}, 34:\penalty0 20132--20145, 2021.

\bibitem[Fujimoto et~al.(2018)Fujimoto, Hoof, and Meger]{fujimoto2018addressing}
Scott Fujimoto, Herke Hoof, and David Meger.
\newblock Addressing function approximation error in actor-critic methods.
\newblock In \emph{International conference on machine learning}, pp.\  1587--1596. PMLR, 2018.

\bibitem[Fujimoto et~al.(2019)Fujimoto, Meger, and Precup]{fujimoto2019off}
Scott Fujimoto, David Meger, and Doina Precup.
\newblock Off-policy deep reinforcement learning without exploration.
\newblock In \emph{International conference on machine learning}, pp.\  2052--2062. PMLR, 2019.

\bibitem[Haydari \& Y{\i}lmaz(2020)Haydari and Y{\i}lmaz]{haydari2020deep}
Ammar Haydari and Yasin Y{\i}lmaz.
\newblock Deep reinforcement learning for intelligent transportation systems: A survey.
\newblock \emph{IEEE Transactions on Intelligent Transportation Systems}, 23\penalty0 (1):\penalty0 11--32, 2020.

\bibitem[Jiang \& Lu(2023)Jiang and Lu]{jiang2023offline}
Jiechuan Jiang and Zongqing Lu.
\newblock Offline decentralized multi-agent reinforcement learning.
\newblock In \emph{ECAI}, pp.\  1148--1155, 2023.

\bibitem[Kalashnikov et~al.(2018)Kalashnikov, Irpan, Pastor, Ibarz, Herzog, Jang, Quillen, Holly, Kalakrishnan, Vanhoucke, et~al.]{kalashnikov2018scalable}
Dmitry Kalashnikov, Alex Irpan, Peter Pastor, Julian Ibarz, Alexander Herzog, Eric Jang, Deirdre Quillen, Ethan Holly, Mrinal Kalakrishnan, Vincent Vanhoucke, et~al.
\newblock Scalable deep reinforcement learning for vision-based robotic manipulation.
\newblock In \emph{Conference on robot learning}, pp.\  651--673. PMLR, 2018.

\bibitem[Kidambi et~al.(2020)Kidambi, Rajeswaran, Netrapalli, and Joachims]{kidambi2020morel}
Rahul Kidambi, Aravind Rajeswaran, Praneeth Netrapalli, and Thorsten Joachims.
\newblock Morel: Model-based offline reinforcement learning.
\newblock \emph{Advances in neural information processing systems}, 33:\penalty0 21810--21823, 2020.

\bibitem[Kostrikov et~al.(2021)Kostrikov, Fergus, Tompson, and Nachum]{kostrikov2021offline}
Ilya Kostrikov, Rob Fergus, Jonathan Tompson, and Ofir Nachum.
\newblock Offline reinforcement learning with fisher divergence critic regularization.
\newblock In \emph{International Conference on Machine Learning}, pp.\  5774--5783. PMLR, 2021.

\bibitem[Kraemer \& Banerjee(2016)Kraemer and Banerjee]{kraemer2016multi}
Landon Kraemer and Bikramjit Banerjee.
\newblock Multi-agent reinforcement learning as a rehearsal for decentralized planning.
\newblock \emph{Neurocomputing}, 190:\penalty0 82--94, 2016.

\bibitem[Kumar et~al.(2019)Kumar, Fu, Soh, Tucker, and Levine]{kumar2019stabilizing}
Aviral Kumar, Justin Fu, Matthew Soh, George Tucker, and Sergey Levine.
\newblock Stabilizing off-policy q-learning via bootstrapping error reduction.
\newblock \emph{Advances in neural information processing systems}, 32, 2019.

\bibitem[Kumar et~al.(2020)Kumar, Zhou, Tucker, and Levine]{kumar2020conservative}
Aviral Kumar, Aurick Zhou, George Tucker, and Sergey Levine.
\newblock Conservative q-learning for offline reinforcement learning.
\newblock \emph{Advances in Neural Information Processing Systems}, 33:\penalty0 1179--1191, 2020.

\bibitem[Lee et~al.(2021)Lee, Jeon, Lee, Pineau, and Kim]{lee2021optidice}
Jongmin Lee, Wonseok Jeon, Byungjun Lee, Joelle Pineau, and Kee-Eung Kim.
\newblock Optidice: Offline policy optimization via stationary distribution correction estimation.
\newblock In \emph{International Conference on Machine Learning}, pp.\  6120--6130. PMLR, 2021.

\bibitem[Lee et~al.(2022)Lee, Paduraru, Mankowitz, Heess, Precup, Kim, and Guez]{lee2022coptidice}
Jongmin Lee, Cosmin Paduraru, Daniel~J Mankowitz, Nicolas Heess, Doina Precup, Kee-Eung Kim, and Arthur Guez.
\newblock Coptidice: Offline constrained reinforcement learning via stationary distribution correction estimation.
\newblock \emph{arXiv preprint arXiv:2204.08957}, 2022.

\bibitem[Levine et~al.(2016)Levine, Finn, Darrell, and Abbeel]{levine2016end}
Sergey Levine, Chelsea Finn, Trevor Darrell, and Pieter Abbeel.
\newblock End-to-end training of deep visuomotor policies.
\newblock \emph{Journal of Machine Learning Research}, 17\penalty0 (39):\penalty0 1--40, 2016.

\bibitem[Levine et~al.(2020)Levine, Kumar, Tucker, and Fu]{levine2020offline}
Sergey Levine, Aviral Kumar, George Tucker, and Justin Fu.
\newblock Offline reinforcement learning: Tutorial, review, and perspectives on open problems.
\newblock \emph{arXiv preprint arXiv:2005.01643}, 2020.

\bibitem[Li et~al.(2023)Li, Hu, Xu, Liu, Zhan, and Zhang]{li2023proto}
Jianxiong Li, Xiao Hu, Haoran Xu, Jingjing Liu, Xianyuan Zhan, and Ya-Qin Zhang.
\newblock Proto: Iterative policy regularized offline-to-online reinforcement learning.
\newblock \emph{arXiv preprint arXiv:2305.15669}, 2023.

\bibitem[Li et~al.(2022)Li, Tang, Tomizuka, and Zhan]{li2022dealing}
Jinning Li, Chen Tang, Masayoshi Tomizuka, and Wei Zhan.
\newblock Dealing with the unknown: Pessimistic offline reinforcement learning.
\newblock In \emph{Conference on Robot Learning}, pp.\  1455--1464. PMLR, 2022.

\bibitem[Mao et~al.(2024)Mao, Xu, Zhang, and Zhan]{mao2024odice}
Liyuan Mao, Haoran Xu, Weinan Zhang, and Xianyuan Zhan.
\newblock Odice: Revealing the mystery of distribution correction estimation via orthogonal-gradient update.
\newblock \emph{arXiv preprint arXiv:2402.00348}, 2024.

\bibitem[Matsunaga et~al.(2023)Matsunaga, Lee, Yoon, Leonardos, Abbeel, and Kim]{matsunaga2023alberdice}
Daiki~E Matsunaga, Jongmin Lee, Jaeseok Yoon, Stefanos Leonardos, Pieter Abbeel, and Kee-Eung Kim.
\newblock Alberdice: addressing out-of-distribution joint actions in offline multi-agent rl via alternating stationary distribution correction estimation.
\newblock \emph{Advances in Neural Information Processing Systems}, 36:\penalty0 72648--72678, 2023.

\bibitem[Matsushima et~al.(2020)Matsushima, Furuta, Matsuo, Nachum, and Gu]{matsushima2020deployment}
Tatsuya Matsushima, Hiroki Furuta, Yutaka Matsuo, Ofir Nachum, and Shixiang Gu.
\newblock Deployment-efficient reinforcement learning via model-based offline optimization.
\newblock \emph{arXiv preprint arXiv:2006.03647}, 2020.

\bibitem[Meng et~al.(2023)Meng, Wen, Le, Li, Xing, Zhang, Wen, Zhang, Wang, Yang, et~al.]{meng2023offline}
Linghui Meng, Muning Wen, Chenyang Le, Xiyun Li, Dengpeng Xing, Weinan Zhang, Ying Wen, Haifeng Zhang, Jun Wang, Yaodong Yang, et~al.
\newblock Offline pre-trained multi-agent decision transformer.
\newblock \emph{Machine Intelligence Research}, 20\penalty0 (2):\penalty0 233--248, 2023.

\bibitem[Nachum \& Dai(2020)Nachum and Dai]{nachum2020reinforcement}
Ofir Nachum and Bo~Dai.
\newblock Reinforcement learning via fenchel-rockafellar duality.
\newblock \emph{arXiv preprint arXiv:2001.01866}, 2020.

\bibitem[Nair et~al.(2020)Nair, Gupta, Dalal, and Levine]{nair2020awac}
Ashvin Nair, Abhishek Gupta, Murtaza Dalal, and Sergey Levine.
\newblock Awac: Accelerating online reinforcement learning with offline datasets.
\newblock \emph{arXiv preprint arXiv:2006.09359}, 2020.

\bibitem[Niu et~al.(2022)Niu, Qiu, Li, Zhou, Hu, Zhan, et~al.]{niu2022trust}
Haoyi Niu, Yiwen Qiu, Ming Li, Guyue Zhou, Jianming Hu, Xianyuan Zhan, et~al.
\newblock When to trust your simulator: Dynamics-aware hybrid offline-and-online reinforcement learning.
\newblock \emph{Advances in Neural Information Processing Systems}, 35:\penalty0 36599--36612, 2022.

\bibitem[Oliehoek et~al.(2008)Oliehoek, Spaan, and Vlassis]{oliehoek2008optimal}
Frans~A Oliehoek, Matthijs~TJ Spaan, and Nikos Vlassis.
\newblock Optimal and approximate q-value functions for decentralized pomdps.
\newblock \emph{Journal of Artificial Intelligence Research}, 32:\penalty0 289--353, 2008.

\bibitem[Pan et~al.(2022)Pan, Huang, Ma, and Xu]{pan2022plan}
Ling Pan, Longbo Huang, Tengyu Ma, and Huazhe Xu.
\newblock Plan better amid conservatism: Offline multi-agent reinforcement learning with actor rectification.
\newblock In \emph{International conference on machine learning}, pp.\  17221--17237. PMLR, 2022.

\bibitem[Peng et~al.(2019)Peng, Kumar, Zhang, and Levine]{peng2019advantage}
Xue~Bin Peng, Aviral Kumar, Grace Zhang, and Sergey Levine.
\newblock Advantage-weighted regression: Simple and scalable off-policy reinforcement learning.
\newblock \emph{arXiv preprint arXiv:1910.00177}, 2019.

\bibitem[Prudencio et~al.(2023)Prudencio, Maximo, and Colombini]{prudencio2023survey}
Rafael~Figueiredo Prudencio, Marcos~ROA Maximo, and Esther~Luna Colombini.
\newblock A survey on offline reinforcement learning: Taxonomy, review, and open problems.
\newblock \emph{IEEE Transactions on Neural Networks and Learning Systems}, 2023.

\bibitem[Rashid et~al.(2020)Rashid, Samvelyan, De~Witt, Farquhar, Foerster, and Whiteson]{rashid2020monotonic}
Tabish Rashid, Mikayel Samvelyan, Christian~Schroeder De~Witt, Gregory Farquhar, Jakob Foerster, and Shimon Whiteson.
\newblock Monotonic value function factorisation for deep multi-agent reinforcement learning.
\newblock \emph{The Journal of Machine Learning Research}, 21\penalty0 (1):\penalty0 7234--7284, 2020.

\bibitem[Samvelyan et~al.(2019)Samvelyan, Rashid, De~Witt, Farquhar, Nardelli, Rudner, Hung, Torr, Foerster, and Whiteson]{samvelyan2019starcraft}
Mikayel Samvelyan, Tabish Rashid, Christian~Schroeder De~Witt, Gregory Farquhar, Nantas Nardelli, Tim~GJ Rudner, Chia-Man Hung, Philip~HS Torr, Jakob Foerster, and Shimon Whiteson.
\newblock The starcraft multi-agent challenge.
\newblock \emph{arXiv preprint arXiv:1902.04043}, 2019.

\bibitem[Shao et~al.(2024)Shao, Qu, Chen, Zhang, and Ji]{shao2024counterfactual}
Jianzhun Shao, Yun Qu, Chen Chen, Hongchang Zhang, and Xiangyang Ji.
\newblock Counterfactual conservative q learning for offline multi-agent reinforcement learning.
\newblock \emph{Advances in Neural Information Processing Systems}, 36, 2024.

\bibitem[Sikchi et~al.(2023)Sikchi, Zhang, and Niekum]{sikchi2023imitation}
Harshit Sikchi, Amy Zhang, and Scott Niekum.
\newblock Imitation from arbitrary experience: A dual unification of reinforcement and imitation learning methods.
\newblock In \emph{Workshop on Reincarnating Reinforcement Learning at ICLR 2023}, 2023.

\bibitem[Silver et~al.(2017)Silver, Schrittwieser, Simonyan, Antonoglou, Huang, Guez, Hubert, Baker, Lai, Bolton, et~al.]{silver2017mastering}
David Silver, Julian Schrittwieser, Karen Simonyan, Ioannis Antonoglou, Aja Huang, Arthur Guez, Thomas Hubert, Lucas Baker, Matthew Lai, Adrian Bolton, et~al.
\newblock Mastering the game of go without human knowledge.
\newblock \emph{nature}, 550\penalty0 (7676):\penalty0 354--359, 2017.

\bibitem[Son et~al.(2019)Son, Kim, Kang, Hostallero, and Yi]{son2019qtran}
Kyunghwan Son, Daewoo Kim, Wan~Ju Kang, David~Earl Hostallero, and Yung Yi.
\newblock Qtran: Learning to factorize with transformation for cooperative multi-agent reinforcement learning.
\newblock In \emph{International conference on machine learning}, pp.\  5887--5896. PMLR, 2019.

\bibitem[Tseng et~al.(2022)Tseng, Wang, Lin, and Isola]{tseng2022offline}
Wei-Cheng Tseng, Tsun-Hsuan~Johnson Wang, Yen-Chen Lin, and Phillip Isola.
\newblock Offline multi-agent reinforcement learning with knowledge distillation.
\newblock \emph{Advances in Neural Information Processing Systems}, 35:\penalty0 226--237, 2022.

\bibitem[Wang et~al.(2020)Wang, Ren, Liu, Yu, and Zhang]{wang2020qplex}
Jianhao Wang, Zhizhou Ren, Terry Liu, Yang Yu, and Chongjie Zhang.
\newblock Qplex: Duplex dueling multi-agent q-learning.
\newblock \emph{arXiv preprint arXiv:2008.01062}, 2020.

\bibitem[Wang et~al.(2022{\natexlab{a}})Wang, Yang, and Wang]{wang2022trust_HAPPO}
Jun Wang, Yaodong Yang, and Zongqing Wang.
\newblock Trust region policy optimization in multi-agent reinforcement learning.
\newblock \emph{arXiv preprint arXiv:2109.11251}, 2022{\natexlab{a}}.

\bibitem[Wang et~al.(2022{\natexlab{b}})Wang, Xu, Zheng, and Zhan]{wang2024offline_OMIGA}
Xiangsen Wang, Haoran Xu, Yinan Zheng, and Xianyuan Zhan.
\newblock Offline multi-agent reinforcement learning with implicit global-to-local value regularization.
\newblock \emph{Advances in Neural Information Processing Systems}, 36, 2022{\natexlab{b}}.

\bibitem[Wu et~al.(2019)Wu, Tucker, and Nachum]{wu2019behavior}
Yifan Wu, George Tucker, and Ofir Nachum.
\newblock Behavior regularized offline reinforcement learning.
\newblock \emph{arXiv preprint arXiv:1911.11361}, 2019.

\bibitem[Xu et~al.(2021)Xu, Zhan, Li, and Yin]{xu2021offline}
Haoran Xu, Xianyuan Zhan, Jianxiong Li, and Honglei Yin.
\newblock Offline reinforcement learning with soft behavior regularization.
\newblock \emph{arXiv preprint arXiv:2110.07395}, 2021.

\bibitem[Xu et~al.(2022{\natexlab{a}})Xu, Jiang, Jianxiong, and Zhan]{xu2022policy}
Haoran Xu, Li~Jiang, Li~Jianxiong, and Xianyuan Zhan.
\newblock A policy-guided imitation approach for offline reinforcement learning.
\newblock \emph{Advances in Neural Information Processing Systems}, 35:\penalty0 4085--4098, 2022{\natexlab{a}}.

\bibitem[Xu et~al.(2022{\natexlab{b}})Xu, Zhan, Yin, and Qin]{xu2022discriminator}
Haoran Xu, Xianyuan Zhan, Honglei Yin, and Huiling Qin.
\newblock Discriminator-weighted offline imitation learning from suboptimal demonstrations.
\newblock In \emph{Proceedings of the 39th International Conference on Machine Learning}, pp.\  24725--24742, 2022{\natexlab{b}}.

\bibitem[Xu et~al.(2022{\natexlab{c}})Xu, Zhan, and Zhu]{xu2022constraints}
Haoran Xu, Xianyuan Zhan, and Xiangyu Zhu.
\newblock Constraints penalized q-learning for safe offline reinforcement learning.
\newblock In \emph{Proceedings of the AAAI Conference on Artificial Intelligence}, volume~36, pp.\  8753--8760, 2022{\natexlab{c}}.

\bibitem[Xu et~al.(2023)Xu, Jiang, Li, Yang, Wang, Chan, and Zhan]{xu2023offline}
Haoran Xu, Li~Jiang, Jianxiong Li, Zhuoran Yang, Zhaoran Wang, Victor Wai~Kin Chan, and Xianyuan Zhan.
\newblock Offline rl with no ood actions: In-sample learning via implicit value regularization.
\newblock \emph{arXiv preprint arXiv:2303.15810}, 2023.

\bibitem[Yang et~al.(2021)Yang, Ma, Li, Zheng, Zhang, Huang, Yang, and Zhao]{yang2021believe}
Yiqin Yang, Xiaoteng Ma, Chenghao Li, Zewu Zheng, Qiyuan Zhang, Gao Huang, Jun Yang, and Qianchuan Zhao.
\newblock Believe what you see: Implicit constraint approach for offline multi-agent reinforcement learning.
\newblock \emph{Advances in Neural Information Processing Systems}, 34:\penalty0 10299--10312, 2021.

\bibitem[Yu et~al.(2022)Yu, Velu, Vinitsky, Gao, Wang, Bayen, and Wu]{yu2022surprising}
Chao Yu, Akash Velu, Eugene Vinitsky, Jiaxuan Gao, Yu~Wang, Alexandre Bayen, and Yi~Wu.
\newblock The surprising effectiveness of ppo in cooperative multi-agent games.
\newblock \emph{Advances in Neural Information Processing Systems}, 35:\penalty0 24611--24624, 2022.

\bibitem[Yu et~al.(2020)Yu, Thomas, Yu, Ermon, Zou, Levine, Finn, and Ma]{yu2020mopo}
Tianhe Yu, Garrett Thomas, Lantao Yu, Stefano Ermon, James~Y Zou, Sergey Levine, Chelsea Finn, and Tengyu Ma.
\newblock Mopo: Model-based offline policy optimization.
\newblock \emph{Advances in Neural Information Processing Systems}, 33:\penalty0 14129--14142, 2020.

\bibitem[Yu et~al.(2021)Yu, Kumar, Rafailov, Rajeswaran, Levine, and Finn]{yu2021combo}
Tianhe Yu, Aviral Kumar, Rafael Rafailov, Aravind Rajeswaran, Sergey Levine, and Chelsea Finn.
\newblock Combo: Conservative offline model-based policy optimization.
\newblock \emph{Advances in neural information processing systems}, 34:\penalty0 28954--28967, 2021.

\bibitem[Zhang et~al.(2022)Zhang, Shao, Jiang, He, Zhang, and Ji]{zhang2022state}
Hongchang Zhang, Jianzhun Shao, Yuhang Jiang, Shuncheng He, Guanwen Zhang, and Xiangyang Ji.
\newblock State deviation correction for offline reinforcement learning.
\newblock In \emph{Proceedings of the AAAI conference on artificial intelligence}, volume~36, pp.\  9022--9030, 2022.

\bibitem[Zhang et~al.(2023)Zhang, Zhang, Xu, Shen, Wang, Chang, Wang, Yuan, and Tao]{zhang2023saformer}
Qin Zhang, Linrui Zhang, Haoran Xu, Li~Shen, Bowen Wang, Yongzhe Chang, Xueqian Wang, Bo~Yuan, and Dacheng Tao.
\newblock Saformer: A conditional sequence modeling approach to offline safe reinforcement learning.
\newblock \emph{arXiv preprint arXiv:2301.12203}, 2023.

\bibitem[Zhang et~al.(2021)Zhang, Li, Wang, Xie, and Lu]{zhang2021fop}
Tianhao Zhang, Yueheng Li, Chen Wang, Guangming Xie, and Zongqing Lu.
\newblock Fop: Factorizing optimal joint policy of maximum-entropy multi-agent reinforcement learning.
\newblock In \emph{International conference on machine learning}, pp.\  12491--12500. PMLR, 2021.

\bibitem[Zheng et~al.(2024)Zheng, Li, Yu, Yang, Li, Zhan, and Liu]{zheng2024safe}
Yinan Zheng, Jianxiong Li, Dongjie Yu, Yujie Yang, Shengbo~Eben Li, Xianyuan Zhan, and Jingjing Liu.
\newblock Safe offline reinforcement learning with feasibility-guided diffusion model.
\newblock \emph{arXiv preprint arXiv:2401.10700}, 2024.

\end{thebibliography}
\bibliographystyle{iclr2025_conference}

\newpage
\appendix
\addtocontents{toc}{\protect\setcounter{tocdepth}{2}}

 \begin{center}
     {\huge APPENDIX}
 \end{center}
Our appendix includes the following:
\begin{itemize}
    \item Proofs of the theoretical claims presented in the main paper.
    \item Details of our experimental settings.
    \item Detailed numerical results from the ablation study investigating the impact of $\alpha$ on ComaDICE's performance.
    \item An ablation study assessing ComaDICE's performance with different forms of f-divergence functions.
    \item An ablation study comparing ComaDICE's performance using 1-layer versus 2-layer mixing networks.
\end{itemize}
\tableofcontents

\newpage
\section{Missing Proofs}

\subsection{Proof of  Proposition \ref{prop.1}}
\noindent  \textbf{Proposition.} \textit{  The   minimax problem in \ref{prob:minimax-w-nu} is equivalent to 
$    \min_{\nu^{tot}} ~~ \left\{\widetilde{\cL}(\nu^{tot})\right\}$,  where
 \[
\widetilde{\cL}(\nu^{tot}) = (1 - \gamma)\bbE_{\bs \sim p_0}[\nu^{tot}(\bs)] + \bbE_{(\bs, \ba) \sim \dmutot}\left[\alpha  f^*\left(\frac{A^{tot}_{\nu}(\bs, \ba)}{\alpha}\right)\right]
 \]
where $f^*$ is  convex conjugate of $f$, i.e., $f^*(y) = \sup_{t\geq 0} \{ty - f(t)\}$. Moreover, if  $\nu^{tot}$ is parameterized by $\theta$, the first order derivative of $\widetilde{\cL}(\nu^{tot})$ w.r.t. $\theta$ is given as
 \[
 \nabla_{\theta}  \widetilde{\cL}(\nu^{tot}) = (1 - \gamma)\bbE_{\bs \sim p_0}[\nabla_\theta \nu^{tot}(\bs)] + \bbE_{(\bs, \ba) \sim \dmutot}\left[\nabla_\theta A^{tot}_{\nu}(\bs, \ba)  w^{tot*}_{\nu}(\bs,\ba) \right]
 \]
where   $w^{tot*}_{\nu}(s,a)  = \max \{0, {f'}^{-1}(A^{tot}_\nu(\bs,\ba)/\alpha)\}$, where $f'^{-1}(\cdot)$ is the inverse function of the first-order derivative of $f$.}

\begin{proof}
We write the Lagrange dual function as:
\begin{align}
    \cL(\nu^{tot}, & \dtot) = \bbE_{(\bo, \ba) \sim \rho^{\bpi_{tot}}}\left[r(\bo, \ba)\right] 
    - \alpha \bbE_{(\bs, \ba) \sim \dmutot}\left[f\left(\frac{\dtot(\bs, \ba)}{\dmutot(\bs, \ba)}\right)\right] \nonumber\\
    & - \sum_{\bs}\nu^{tot}(\bs) \left(\sum_{\ba'} \dtot(\bs, \ba') - (1 - \gamma) p_0(\bs) - \gamma \sum_{\ba', \bs'} \dtot(\bs', \ba') P(\bs | \ba', \bs')\right)\nonumber \\
    & = \sum_{\bs}\nu^{tot}(\bs) (1 - \gamma) p_0(\bs) - \alpha \bbE_{(\bs, \ba) \sim \dmutot}\left[f\left(\frac{\dtot(\bs, \ba)}{\dmutot(\bs, \ba)}\right)\right] \nonumber \\
    & ~~~~~~+\sum_{\bs, \ba}  \dmutot(\bs,\ba)\left(r(\bo,\ba) +  \gamma \bbE_{\bs'\sim P(\cdot|\bs,\ba)}  \nu^{tot}(\bs') - \nu^{tot}(\bs)\right)\nonumber\\
    & = (1 - \gamma)\bbE_{\bs \sim p_0}[\nu^{tot}(\bs)] + \bbE_{(\bs, \ba) \sim \dmutot}\left[-\alpha f\left(w^{tot}_\nu(\bs, \ba)\right) + w^{tot}_\nu(\bs, \ba) A^{tot}_{\nu}(\bs, \ba)\right]
\end{align}
where $w^{tot}_\nu(\bs, \ba) = \frac{\dtot(\bs, \ba)}{\dmutot(\bs, \ba)}$. We now see that, for each $(\bs,\ba)$, each component $-\alpha f\left(w^{tot}_\nu(\bs, \ba)\right) + w^{tot}_\nu(\bs, \ba) A^{tot}_{\nu}(\bs, \ba)$ is maximized at:
\[
\max_{w^{tot}\geq 0} -\alpha f\left(w^{tot}_\nu(\bs, \ba)\right) + w^{tot}_\nu(\bs, \ba) A^{tot}_{\nu}(\bs, \ba) =   f^*\left( \frac{A^{tot}_{\nu}(\bs, \ba) }{\alpha}\right)
\]
where $f^*$ is the (variant) convex conjugate of the convex function $f$. We then obtain:
\[
\max_{w^{tot}\geq 0}  \cL(\nu^{tot},w^{tot}) = \widetilde{\cL}(\nu^{tot}) =   (1 - \gamma)\bbE_{\bs \sim p_0}[\nu^{tot}(\bs)] + \bbE_{(\bs, \ba) \sim \dmutot}\left[\alpha  f^*\left(\frac{A^{tot}_{\nu}(\bs, \ba)}{\alpha}\right)\right]
\]
Moreover, consider the maximization problem $ \max_{w^{tot}\geq 0} T(w^{tot}(\bs,\ba)) =  -\alpha f\left(w^{tot}_\nu(\bs, \ba)\right) + w^{tot}_\nu(\bs, \ba) A^{tot}_{\nu}(\bs, \ba)$. Taking its first-order derivative w.r.t $w^{tot}(\bs,\ba)$ yields:
\[
-\alpha f'(w^{tot}(\bs,\ba)) + A^{tot}_{\nu}(\bs, \ba)
\]
So, if  $f'^{-1}\left(\frac{A^{tot}_{\nu}(\bs, \ba)}{\alpha}\right) \geq 0$, then  $w^{tot*}(\bs,\ba) = f'^{-1}\left(\frac{A^{tot}_{\nu}(\bs, \ba)}{\alpha}\right) \geq 0$ is optimal for the maximization problem. Otherwise, if $ f'^{-1}\left(\frac{A^{tot}_{\nu}(\bs, \ba)}{\alpha}\right) < 0$, we see that $T(w^{tot}(\bs,\ba))$ is increasing when $w^{tot}(\bs,\ba) \leq f'^{-1}\left(\frac{A^{tot}_{\nu}(\bs, \ba)}{\alpha}\right)$ and decreasing when $w^{tot}(\bs,\ba) \geq f'^{-1}\left(\frac{A^{tot}_{\nu}(\bs, \ba)}{\alpha}\right)$, implying that the maximization problem has an optimal solution at $w^{tot*}(\bs,\ba) = 0$. So, putting all together, $w^{tot*}_{\nu}(s,a)  = \max \{0, {f'}^{-1}(A^{tot}_\nu(\bs,\ba)/\alpha)\}$  is optimal for the maximization problem $ \max_{w^{tot}\geq 0} T(w^{tot}(\bs,\ba))$.

To get derivatives of $\widetilde{\cL}(\nu^{tot})$, we note that, for any $y \in \bbR$, $\nabla f^*(y) =   t^*$, where $y^* = \text{argmax}_{t\geq 0} (ty -  f(t))$. Thus, the first-order derivative of $f^*\left(\frac{A^{tot}_{\nu}(\bs, \ba)}{\alpha}\right)$ can be computed as:
\[
\nabla_\theta f^*\left(\frac{A^{tot}_{\nu}(\bs, \ba)}{\alpha}\right) =  \frac{\nabla_\theta A^{tot}_{\nu}(\bs, \ba)}{\alpha} w^{tot^*}(\bs,\ba)
\]
which implies:
\[
 \nabla_{\theta}  \widetilde{\cL}(\nu^{tot}) = (1 - \gamma)\bbE_{\bs \sim p_0}[\nabla_\theta \nu^{tot}(\bs)] + \bbE_{(\bs, \ba) \sim \dmutot}\left[\nabla_\theta A^{tot}_{\nu}(\bs, \ba) w^{tot*}_{\nu}(\bs,\ba) \right]
\]
we complete the proof.
\end{proof}

\subsection{Proof of Theorem \ref{th:convex} }
\noindent \textbf{Theorem}. Assume the mixing network $\cM_\theta[\cdot]$ is constructed with non-negative weights and convex activations, then $\widetilde{\cL}(\bnu,\theta)$ is convex in $\bnu$.

\begin{proof}
    We first  introduce the following lemma, which is essential to validate the convexity of $\widetilde{\cL}(\bnu, \theta)$.
    \begin{lemma}\label{lemma.1}
        If the mixing network are multi-level feed-forward, constructed with non-negative  weights   and convex activations, then $\cM_\theta [\bnu(\bs)]$  and $\cM_\theta[\bq(\bs,\ba) - \bnu(\bs)]$  are convex in $\bnu$ 
    \end{lemma}
    \begin{proof}
To simplify the proof, we first prove a general result stating that if $\cM_\theta [\bX]$ is a multi-level feed-forward network with non-negative weights and convex activations, then $\cM_\theta [\bX]$ is convex in $\bX$. To start, we note that any $N$-layer feed-forward network with input $\bX$ can be defined recursively as
\begin{align}
    F^0(\bX) &= \bX \\
    F^n(\bX) &= \sigma^n\Big(F^{n-1}(\bX)\Big) \times W_n + b_n,~ n = 1,\ldots,N
\end{align}
where $\sigma^n$ is a set of activation functions applied to each element of vector $F^{n-1}(\bX)$, and $W_n$ and $b_n$ are the weights and biases, respectively, at layer $n$. Therefore, we will prove the result by induction, i.e., $F^n(\bX)$ is convex and non-decreasing in $\bX$ for $n = 0,\ldots$. Here we note that $F^n(\bX)$ is a vector, so when we say ``\textit{$F^n(\bX)$ is convex and non-decreasing in $\bX$},'' it means each element of $F^n(\bX)$ is convex and non-decreasing in $\bX$.

We first see that the claim indeed holds for $n = 0$. Now let us assume that $F^{n-1}(\bX)$ is convex and non-decreasing in $\bX$; we will prove that $F^n(\bX)$ is also convex and non-decreasing in $\bX$. The non-decreasing property can be easily verified as we can see, given two vectors $\bX$ and $\bX'$ such that $\bX \geq \bX'$ (element-wise comparison), we have the following chain of inequalities:
\begin{align*}
    F^{n-1}(\bX) &\stackrel{(a)}{\geq} F^{n-1}(\bX') \\
    \sigma^n(F^{n-1}(\bX)) &\stackrel{(b)}{\geq} \sigma^n(F^{n-1}(\bX')) \\
    \sigma^n(F^{n-1}(\bX)) \times W_n + b_n &\stackrel{(c)}{\geq} \sigma^n(F^{n-1}(\bX')) \times W_n + b_n
\end{align*}

where $(a)$ is due to the induction assumption that $F^{n-1}(\bX)$ is non-decreasing in $\bX$, $(b)$ is because $\sigma^n$ is also non-decreasing, and $(c)$ is because the weights $W_n$ are non-negative.

To verify the convexity of $F^n(\bX)$, we will show that for any $\bX, \bX'$, and any scalar $\alpha \in (0,1)$, the following holds:
\begin{equation}
    \alpha F^n(\bX) + (1-\alpha) F^n(\bX) \geq F^n(\alpha \bX + (1-\alpha) \bX')
\end{equation}
To this end, we write:
\begin{align*}
    \alpha F^n(\bX) + (1-\alpha) F^n(\bX') &= \Big(\alpha \sigma^n(F^{n-1}(\bX)) + (1-\alpha) \sigma^n(F^{n-1}(\bX'))\Big) \times W_n + b_n \\
    &\stackrel{(d)}{\geq} \Big(\sigma^n\Big(\alpha F^{n-1}(\bX) + (1-\alpha) F^{n-1}(\bX')\Big)\Big) \times W_n + b_n \\
    &\stackrel{(e)}{\geq} \Big(\sigma^n\Big(F^{n-1}(\alpha \bX + (1-\alpha) \bX')\Big)\Big) \times W_n + b_n \\
    &= F^n(\alpha \bX + (1-\alpha) \bX')
\end{align*}
where $(d)$ is due to the assumption that activation functions $\sigma^n$ are convex and $W_n \geq 0$, and $(e)$ is because $\alpha F^{n-1}(\bX) + (1-\alpha) F^{n-1}(\bX') \geq F^{n-1}(\alpha \bX + (1-\alpha) \bX')$ (because $F^{n-1}(\bX)$ is convex in $\bX$, by the induction assumption), and the activation functions $\sigma^n$ are non-decreasing and $W_n \geq 0$. So, we have:
\[
    \alpha F^n(\bX) + (1-\alpha) F^n(\bX') \geq F^n(\alpha \bX + (1-\alpha) \bX')
\]
implying that $F^n(\bX)$ is convex in $\bX$. We then complete the induction proof and conclude that $F^n(\bX)$ is convex and non-decreasing in $\bX$ for any $n = 0, \ldots, N$.

From the result above, since both $\bnu(\bs)$ and $\bq(\bs,\ba) - \bnu(\bs)$ are linear in $\bnu$, it follows that $\cM_\theta[\bnu(\bs)]$ and $\cM_\theta[\bq(\bs,\ba) - \bnu(\bs)]$ are convex with respect to $\bnu$.
\end{proof}

We are now ready to prove the convexity of $\widetilde{\cL}(\bnu, \theta)$ with respect to $\bnu$. Directly verifying the convexity of this function is challenging, as it involves some complicated components such as $f^*\left(\frac{\cM_{\theta} [\bq(\bs,\ba) - \bnu(\bs)]}{\alpha}\right)$, which is difficult to analyze. However, we recall that:
\[
\widetilde{\cL}(\bnu, \theta) = \max_{w^{tot} \geq 0} \cL(\bnu, \theta, w^{tot}),
\]
where
\[
\begin{aligned}
    \cL(\bnu, \theta, w^{tot}) &=  (1 - \gamma)\bbE_{\bs \sim p_0}[\cM_\theta[\bnu(\bs)]] \\
    &\quad + \bbE_{(\bs, \ba) \sim \dmutot}\left[-\alpha f\left(w^{tot}_\nu(\bs, \ba)\right) + w^{tot}_\nu(\bs, \ba) \cM_\theta[\bq(\bs, \ba) - \bnu(\bs)]\right].
\end{aligned}
\]
From Lemma~\ref{lemma.1}, we know that $\cM_\theta[\bnu(\bs)]$ and $\cM_\theta[\bq(\bs, \ba) - \bnu(\bs)]$ are convex in $\bnu$, thus $\cL(\bnu, \theta, w^{tot})$ is also convex in $\bnu$. We now follow the standard approach to verify the convexity of $\widetilde{\cL}(\bnu, \theta)$ as follows. Let $\bnu^1$ and $\bnu^2$ be two feasible value functions. Given any $\beta \in (0,1)$, we will prove that:
\begin{equation}
   \beta \widetilde{\cL}(\bnu^1, \theta) + (1-\beta) \widetilde{\cL}(\bnu^2, \theta) \geq \widetilde{\cL}(\beta\bnu^1+(1-\beta)\bnu^2, \theta).\label{eq:cv}
\end{equation}
To see why this should hold, we recall that $\cL(\bnu, \theta, w^{tot})$ is convex in $\bnu$ and $\widetilde{\cL}(\bnu, \theta) = \max_{w^{tot} \geq 0} \cL(\bnu, \theta, w^{tot})$, leading to the following chain of inequalities:
\begin{align}
   \beta \widetilde{\cL}(\bnu^1, \theta) + (1-\beta) \widetilde{\cL}(\bnu^2, \theta) &= \beta \max_{w^{tot}} \cL(\bnu^1, \theta, w^{tot}) + (1-\beta) \max_{w^{tot}} \cL(\bnu^2, \theta, w^{tot}) \nonumber\\
   &\geq \max_{w^{tot}} \left\{\beta \cL(\bnu^1, \theta, w^{tot}) + (1-\beta) \cL(\bnu^2, \theta, w^{tot})\right\} \nonumber\\
   &\geq \max_{w^{tot}} \left\{\cL(\beta\bnu^1 + (1-\beta)\bnu^2, \theta, w^{tot})\right\} \nonumber\\
   &= \widetilde{\cL}(\beta\bnu^1 + (1-\beta)\bnu^2, \theta). \nonumber
\end{align}
The last inequality directly confirms Eq. \ref{eq:cv}, implying the convexity of $\widetilde{\cL}(\bnu, \theta)$ in $\bnu$, as desired.
\end{proof}

\subsection{Proof of  Proposition \ref{prop.consistency}}

\noindent \textbf{Proposition.}
  \textit{  Let \( \pi^* \) be the optimal solution to \ref{prob:pi-local}. Then \( \boldsymbol{\pi}^*_{tot}(\mathbf{a}|\mathbf{s}) = \prod_{i\in \mathcal{N}} \pi^*_i(a_i|o_i) \) is also optimal for the global weighted BC problem \ref{prob:pi-tot}.
}

\begin{proof}
 To prove that \( \boldsymbol{\pi}^*_{tot}(\mathbf{a}|\mathbf{s}) = \prod_{i\in \mathcal{N}} \pi^*_i(a_i|o_i) \) is optimal for the global WBC problem \ref{prob:pi-tot}, we need to verify that 
\[
\mathbb{E}_{(\bs,\ba) \sim \dmutot} \left[ w^{tot*}(\bs,\ba) \log \bpi_{tot}(\ba|\bs) \right] \leq \mathbb{E}_{(\bs,\ba) \sim \dmutot} \left[ w^{tot*}(\bs,\ba) \log \bpi^*_{tot}(\ba|\bs) \right]
\]
for any global policy \( \bpi_{tot} \in \Pi_{tot} \).

Since \( \bpi_{tot} \) is decomposable, there exist local policies \( \pi_i \) such that
\[
\bpi_{tot}(\ba|\bs) = \prod_{i\in \cN} \pi_i(a_i|o_i).
\]
As a result, we have the following inequalities:
\begin{align*}
    \mathbb{E}_{(\bs,\ba) \sim \dmutot} \left[ w^{tot*}(\bs,\ba) \log \bpi_{tot}(\ba|\bs) \right] 
    &= \mathbb{E}_{(\bs,\ba) \sim \dmutot} \left[ w^{tot*}(\bs,\ba) \sum_{i\in \cN} \log \pi_i(a_i|o_i) \right] \\
    &= \sum_{i\in \cN} \mathbb{E}_{(\bs,\ba) \sim \dmutot} \left[ w^{tot*}(\bs,\ba) \log \pi_i(a_i|o_i) \right] \\
    &\leq \sum_{i\in \cN} \max_{\pi'_i} \mathbb{E}_{(\bs,\ba) \sim \dmutot} \left[ w^{tot*}(\bs,\ba) \log \pi'_i(a_i|o_i) \right] \\
    &= \sum_{i\in \cN} \mathbb{E}_{(\bs,\ba) \sim \dmutot} \left[ w^{tot*}(\bs,\ba) \log \pi^*_i(a_i|o_i) \right] \\
    &= \mathbb{E}_{(\bs,\ba) \sim \dmutot} \left[ w^{tot*}(\bs,\ba) \log \bpi^*_{tot}(\ba|\bs) \right],
\end{align*}
which directly implies that \( \bpi^*_{tot} \) is optimal for the global WBC problem \ref{prob:pi-tot}.

\end{proof}

\clearpage
\section{Additional Details}

\subsection{Offline Multi-Agent Datasets}

\begin{table}[H]
\small
\centering
\resizebox{\textwidth}{!}{
\begin{tabular}{cc|ccccccc}
\toprule
\multicolumn{2}{c|}{\multirow{2}{*}{\textbf{Instances}}} & \multirow{2}{*}{\textbf{Trajectories}} & \multirow{2}{*}{\textbf{Samples}} & \multirow{2}{*}{\textbf{Agents}} & \textbf{State} & \textbf{Obs} & \textbf{Action} & \textbf{Average} \\
\multicolumn{2}{c|}{} &  &  &  & \textbf{dim} & \textbf{dim} & \textbf{dim} & \textbf{returns} \\
\midrule
\multirow{3}{*}{2c\_vs\_64zg} & poor & 0.3K & 21.7K & 2 & 675 & 478 & 70 & 8.9±1.0 \\
    & medium & 1.0K & 75.9K & 2 & 675 & 478 & 70 & 13.0±1.4 \\
    & good & 1.0K & 118.4K & 2 & 675 & 478 & 70 & 19.9±1.3 \\
\midrule
\multirow{3}{*}{5m\_vs\_6m} & poor & 1.0K & 113.7K & 5 & 156 & 124 & 12 & 8.5±1.2 \\
    & medium & 1.0K & 138.6K & 5 & 156 & 124 & 12 & 11.0±0.6 \\
    & good & 1.0K & 138.7K & 5 & 156 & 124 & 12 & 20.0±0.0 \\
\midrule
\multirow{3}{*}{6h\_vs\_8z} & poor & 1.0K & 145.5K & 6 & 213 & 172 & 14 & 9.1±0.8 \\
    & medium & 1.0K & 177.1K & 6 & 213 & 172 & 14 & 12.0±1.3 \\
    & good & 1.0K & 228.2K & 6 & 213 & 172 & 14 & 17.8±2.1 \\
\midrule
\multirow{3}{*}{corridor} & poor & 1.0K & 307.6K & 6 & 435 & 346 & 30 & 4.9±1.7 \\
    & medium & 1.0K & 756.1K & 6 & 435 & 346 & 30 & 13.1±1.3 \\
    & good & 1.0K & 601.0K & 6 & 435 & 346 & 30 & 19.9±1.0 \\
\midrule
\multirow{5}{*}{Protoss} & 5\_vs\_5 & 1.0K & 60.8K & 5 & 130 & 92 & 11 & 16.8±6.3 \\
    & 10\_vs\_10 & 1.0K & 68.3K & 10 & 310 & 182 & 16 & 15.7±5.2 \\
    & 10\_vs\_11 & 1.0K & 62.9K & 10 & 327 & 191 & 17 & 15.3±5.7 \\
    & 20\_vs\_20 & 1.0K & 76.7K & 20 & 820 & 362 & 26 & 16.2±4.7 \\
    & 20\_vs\_23 & 1.0K & 65.0K & 20 & 901 & 389 & 29 & 14.0±4.5 \\
\midrule
\multirow{5}{*}{Terran} & 5\_vs\_5 & 1.0K & 47.6K & 5 & 120 & 82 & 11 & 15.2±7.2 \\
    & 10\_vs\_10 & 1.0K & 56.4K & 10 & 290 & 162 & 16 & 14.7±6.2 \\
    & 10\_vs\_11 & 1.0K & 52.5K & 10 & 306 & 170 & 17 & 12.1±5.7 \\
    & 20\_vs\_20 & 1.0K & 63.0K & 20 & 780 & 322 & 26 & 14.0±6.0 \\
    & 20\_vs\_23 & 1.0K & 51.3K & 20 & 858 & 346 & 29 & 11.7±5.7 \\
\midrule
\multirow{5}{*}{Zerg} & 5\_vs\_5 & 1.0K & 27.5K & 5 & 120 & 82 & 11 & 10.4±5.0 \\
    & 10\_vs\_10 & 1.0K & 31.9K & 10 & 290 & 162 & 16 & 14.7±6.0 \\
    & 10\_vs\_11 & 1.0K & 30.9K & 10 & 306 & 170 & 17 & 12.0±5.1 \\
    & 20\_vs\_20 & 1.0K & 35.4K & 20 & 780 & 322 & 26 & 12.3±4.2 \\
    & 20\_vs\_23 & 1.0K & 32.8K & 20 & 858 & 346 & 29 & 10.8±4.0 \\
\midrule
\multirow{4}{*}{Hopper} & expert & 1.5K & 999K & 3 & 42 & 14 & 1 & 2452.0±1097.9 \\
    & medium & 4.0K & 915K & 3 & 42 & 14 & 1 & 723.6±211.7 \\
    & m-replay & 4.2K & 1311K & 3 & 42 & 14 & 1 & 746.4±671.9 \\
    & m-expert & 5.5K & 1914K & 3 & 42 & 14 & 1 & 1190.6±973.4 \\
\midrule
\multirow{4}{*}{Ant} & expert & 1.0K & 1000K & 2 & 226 & 113 & 4 & 2055.1±22.1 \\
    & medium & 1.0K & 1000K & 2 & 226 & 113 & 4 & 1418.7±37.0 \\
    & m-replay & 1.8K & 1750K & 2 & 226 & 113 & 4 & 1029.5±141.3 \\
    & m-expert & 2.0K & 2000K & 2 & 226 & 113 & 4 & 1736.9±319.6 \\
\midrule
\multirow{4}{*}{\shortstack{Half\\Cheetah}} & expert & 1.0K & 1000K & 6 & 138 & 23 & 1 & 2785.1±1053.1 \\
    & medium & 1.0K & 1000K & 6 & 138 & 23 & 1 & 1425.7±520.1 \\
    & m-replay & 1.0K & 1000K & 6 & 138 & 23 & 1 & 655.8±590.4 \\
    & m-expert & 2.0K & 2000K & 6 & 138 & 23 & 1 & 2105.4±1073.2 \\
\midrule
\end{tabular}
}
\caption{Overview of datasets used in experiments, including details of trajectories, samples, agent counts, and state, observation, and action space dimensions across SMACv1, SMACv2, and MaMujoco environments, with average returns indicating performance levels.}
\label{data_infor}
\end{table}

\subsection{Implementation Details}
% This section shows the implementation details
% All experiments are implemented with Pytorch and run parallel on a single NVIDIA® H100 NVL Tensor Core GPU. Number of sub-tasks we need to run for main experiments are 39 (instances) x 7 (algos) x 5 (seeds)) = 1365.
% More that that, the offline dataset of each instance are too large, up to 7.4 GB of drive size. We create a preprocessing step, includes: read all transitions of each dataset, combine single file of each trajectory into one large numpy object contains batch of trajectory, define data type of each element like states (float32), actions (int64), dones (bool), etc, and save it into compressed numpy file.
% Although it significantly increases the computing performance, but we still need a huge number of RAM for loading the whole dataset.
% We run parrallel due to the computation cost are so high, that's why we cannot report the training time for each method.
Our experiments were implemented using PyTorch and executed in parallel on a single NVIDIA® H100 NVL Tensor Core GPU. Our study required running a large number of sub-tasks, specifically 1,365 in total (i.e., 39 instances across 7 algorithms with 5 different random seeds each). 
% This parallelization was crucial due to the high computational cost, which also prevents us from reporting the exact training time for each method.

\begin{figure}[H]
  \centering
  \includegraphics[width=0.99\textwidth]{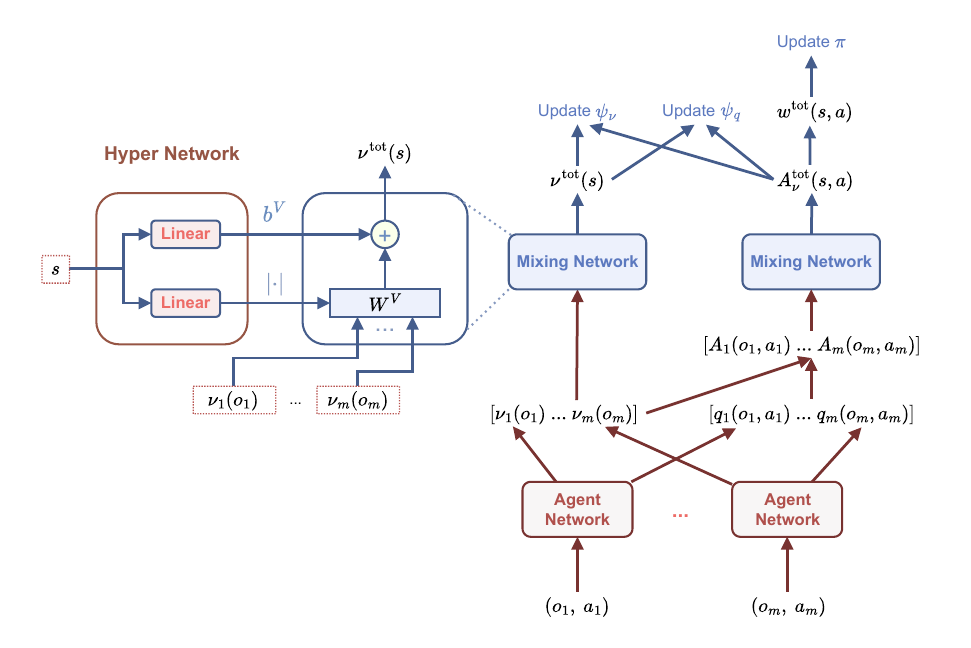}
  \caption{Our ComaDICE model architecture.}
\end{figure}

The offline datasets for each instance are substantial, reaching sizes of up to 7.4 GB. To manage this, we developed a preprocessing step designed to optimize data handling and improve computational efficiency. This process involves reading all transitions from each dataset and combining individual trajectory files into a single large NumPy object that contains batches of trajectories. In this step, we define the data type for each element, such as states (float32), actions (int64), and dones (bool), ensuring consistent and efficient data storage. The processed data is then saved into a compressed NumPy file, which significantly boosts computing performance.

Despite these optimizations, loading the entire dataset still requires a large amount of RAM. By leveraging parallel processing and efficient data management strategies, we effectively managed the extensive computational and memory demands of our experiments. This approach allowed us to handle the large-scale data and complex computations necessary for our study.

\subsubsection{Hyper-parameters}
% We create two versions of our MADice for each domain: continuous (MaMujoco) and discrete (SMACv1 & SMACv2).
% The key differents are distribution. We use Gaussian distribution (torch.distributions.Normal) for continuous, and Categorical distribution (torch.distributions.Categorical) for discrete. Note that in discrete version, we compute probability of each action of each agent by calculating the softmax of all available actions of the current agent, not all actions. In other words, probability of the action that not in available actions is zero. It helps the actor calculate the log likelihood more correct.
% Hyperparameter Value	
% Optimizer	Adam
% Learning rate (for both Q-value and policy networks)	1e-4
% Tau (soft update target rate)	0.005
% Gamma (discount factor)	0.99
% Batch size	128
% Agent hidden dimension	256
% Mixer hidden dimension	64
% Number of seeds	5
% Number of episodes for each evaluation step	32
% Number of evaluation steps	100
% Lambda scale	1.0
% Alpha	10
% f-divergence	soft-X^2

\begin{table}[H]
\centering
\begin{tabular}{ll}
\toprule
\textbf{Hyperparameter} & \textbf{Value} \\
\midrule
Optimizer & Adam \\
Learning rate (Q-value and policy networks) & \(1 \times 10^{-4}\) \\
Tau (\(\tau\)) & 0.005 \\
Gamma (\(\gamma\)) & 0.99 \\
Batch size & 128 \\
Agent hidden dimension & 256 \\
Mixer hidden dimension & 64 \\
Number of seeds & 5 \\
Number of episodes per evaluation step & 32 \\
Number of evaluation steps & 100 \\
Lambda scale (\(\lambda\)) & 1.0 \\
Alpha (\(\alpha\)) & 10 \\
f-divergence & soft-$\chi^2$ \\
\bottomrule
\end{tabular}
\caption{Hyperparameters for our algorithm}
\label{hyperparameters}
\end{table}

In our study, we developed two versions of our algorithm: a continuous version for MaMujoco using Gaussian distributions (torch.distributions.Normal), and a discrete version for SMACv1 and SMACv2 using Categorical distributions (torch.distributions.Categorical). In the discrete setting, action probabilities are computed using softmax over available actions only, ensuring zero probability for unavailable actions, which enhances the accuracy of log likelihood calculations. Key hyperparameters are listed on the Table \ref{hyperparameters}. Experiments were conducted with 5 seeds, 32 episodes per evaluation step, and 100 evaluation steps.

\subsection{Additional Experimental Details}
% We report two metric for agent evaluation: mean/std returns and mean/std winrates
% Note that winrates are available for competitive environments only (SMACv1 and SMACv2), not available for MaMujoco
% Each score is calculated by average/std of final results of all 5 seeds
% We also show the figures contains evaluation curves of each method during training agent with offline dataset.
We evaluate the performance of our ComaDICE algorithm using two key metrics: mean and standard deviation (std) of returns and winrates. Returns measure the average rewards accumulated by agents, calculated across five random seeds to ensure robustness, while winrates, applicable only to competitive environments like SMACv1 and SMACv2, indicate the success rate against other agents. For cooperative settings such as MaMujoco, winrates are not applicable. We also include figures showing evaluation curves, highlighting how each method's performance evolves during training with offline datasets. These metrics and visualizations provide a comprehensive overview of our algorithm's effectiveness and consistency in various MARL tasks.

\subsubsection{Returns}

Tables \ref{SMACv1:return}, \ref{SMACv2:return}, and \ref{MaMujoco:return} present the returns from our experimental results across the SMACv1, SMACv2, and Multi-Agent MuJoCo environments, highlighting the performance of our proposed algorithm, ComaDICE, alongside baseline methods such as BC, BCQ, CQL, ICQ, OMAR, and OMIGA. Our results demonstrate that ComaDICE consistently achieves superior returns, particularly excelling in more complex difficulty tasks. Figures \ref{SMACv1:curves:returns}, \ref{SMACv2:curves:returns}, and \ref{MaMujoco:curves:returns} illustrate the learning curves for these algorithms, showing that ComaDICE not only outperforms other algorithms in terms of mean returns but also exhibits lower standard deviation, indicating robust and stable performance. This suggests that ComaDICE effectively handles distributional shifts in offline settings. These findings underscore our algorithm's adaptability and effectiveness in diverse multi-agent coordination scenarios, setting a new benchmark in offline MARL.

\begin{table}[H]
\small
\centering
\resizebox{\textwidth}{!}{
\begin{tabular}{cc|ccccccc}
\toprule
\multicolumn{2}{c|}{\textbf{Instances}} & \textbf{BC} & \textbf{BCQ} & \textbf{CQL} & \textbf{ICQ} & \textbf{OMAR} & \textbf{OMIGA} & \textbf{ComaDICE} \\
\midrule
\multirow{3}{*}{2c\_vs\_64zg} & poor & 11.6±0.4 & 12.5±0.2 & 10.8±0.5 & 12.6±0.2 & 11.3±0.5 & \red{13.0±0.7} & 12.1±0.5 \\
 & medium & 13.4±1.9 & 15.6±0.4 & 12.8±1.6 & 15.6±0.6 & 10.2±0.2 & 16.0±0.2 & \red{16.3±0.7} \\
 & good & 17.9±1.3 & 19.1±0.3 & 18.5±1.0 & 18.8±0.2 & 17.3±0.8 & 19.1±0.3 & \red{20.3±0.1} \\
\midrule
\multirow{3}{*}{5m\_vs\_6m} & poor & 7.0±0.5 & 7.6±0.4 & 7.4±0.1 & 7.3±0.2 & 7.3±0.4 & 7.5±0.2 & \red{8.1±0.5} \\
 & medium & 7.0±0.8 & 7.6±0.1 & 7.8±0.1 & 7.8±0.3 & 7.1±0.5 & 7.9±0.6 & \red{8.7±0.4} \\
 & good & 7.0±0.5 & 7.8±0.1 & 8.1±0.2 & 7.9±0.3 & 7.4±0.6 & 8.3±0.4 & \red{8.7±0.5} \\
\midrule
\multirow{3}{*}{6h\_vs\_8z} & poor & 8.6±0.8 & 10.8±0.2 & 10.8±0.5 & 10.6±0.1 & 10.6±0.2 & 11.3±0.2 & \red{11.4±0.6} \\
 & medium & 9.5±0.3 & 11.8±0.2 & 11.3±0.3 & 11.1±0.3 & 10.4±0.2 & 12.2±0.2 & \red{12.8±0.2} \\
 & good & 10.0±1.7 & 12.2±0.2 & 10.4±0.2 & 11.8±0.1 & 9.9±0.3 & 12.5±0.2 & \red{13.1±0.5} \\
\midrule
\multirow{3}{*}{corridor} & poor & 2.9±0.6 & 4.5±0.9 & 4.1±0.6 & 4.5±0.3 & 4.3±0.5 & 5.6±0.3 & \red{6.4±0.5} \\
 & medium & 7.4±0.8 & 10.8±0.9 & 7.0±0.7 & 11.3±1.6 & 7.3±0.7 & 11.7±1.3 & \red{12.9±0.6} \\
 & good & 10.8±2.6 & 15.2±1.2 & 5.2±0.8 & 15.5±1.1 & 6.7±0.7 & 15.9±0.9 & \red{18.0±0.1} \\
\bottomrule
\end{tabular}
}
\caption{Comparison of average returns for ComaDICE and baselines on SMACv1 benchmarks.}
\label{SMACv1:return}
\end{table}

\begin{figure}[H]\small
\small
\centering
\showlegend[0.8]{10}{7}
\\
\showinstance[0.16]{0}{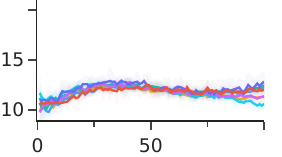}{2c\_vs\_64zg \\ \emph{poor}}
\showinstance[0.15]{12}{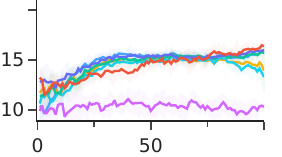}{2c\_vs\_64zg \\ \emph{medium}}
\showinstance[0.15]{12}{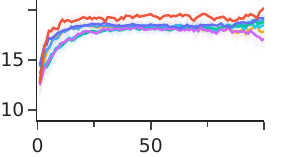}{2c\_vs\_64zg \\ \emph{good}}
\showinstance[0.16]{0}{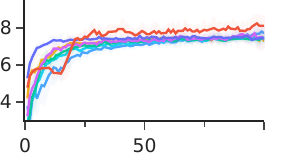}{5m\_vs\_6m \\ \emph{poor}}
\showinstance[0.15]{8}{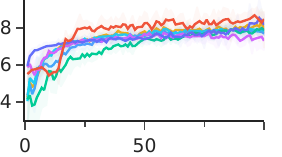}{5m\_vs\_6m \\ \emph{medium}}
\showinstance[0.15]{8}{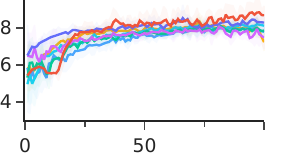}{5m\_vs\_6m \\ \emph{good}}
\\
\showinstance[0.16]{0}{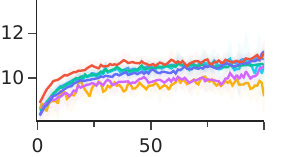}{6h\_vs\_8z \\ \emph{poor}}
\showinstance[0.15]{12}{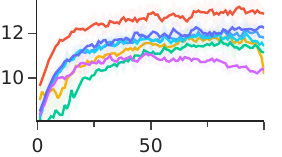}{6h\_vs\_8z \\ \emph{medium}}
\showinstance[0.15]{12}{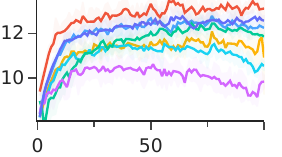}{6h\_vs\_8z \\ \emph{good}}
\showinstance[0.16]{0}{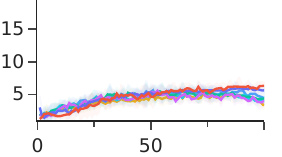}{corridor \\ \emph{poor}}
\showinstance[0.15]{12}{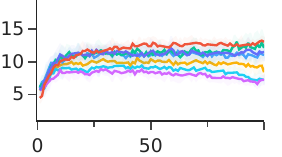}{corridor \\ \emph{medium}}
\showinstance[0.15]{12}{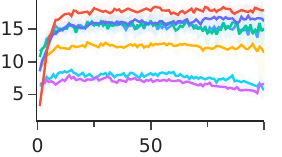}{corridor \\ \emph{good}}
\\
\caption{Evaluation of SMACv1 tasks comparing the returns achieved by ComaDICE and baselines.}
\label{SMACv1:curves:returns}
\end{figure}

\begin{table}[H]
\small
\centering
\begin{tabular}{cc|ccccccc}
\toprule
\multicolumn{2}{c|}{\textbf{Instances}} & \textbf{BC} & \textbf{BCQ} & \textbf{CQL} & \textbf{ICQ} & \textbf{OMAR} & \textbf{OMIGA} & \textbf{ComaDICE} \\
\midrule
\multirow{5}{*}{Protoss} & 5\_vs\_5 & 13.2±0.7 & 6.8±1.6 & 9.3±1.6 & 10.7±1.2 & 8.9±0.8 & 14.3±1.4 & \red{14.4±1.1} \\
 & 10\_vs\_10 & 12.0±1.9 & 7.7±1.3 & 11.3±0.9 & 10.4±1.6 & 8.8±0.6 & 14.2±1.5 & \red{14.6±1.8} \\
 & 10\_vs\_11 & 11.2±0.5 & 5.2±1.4 & 7.9±0.8 & 10.3±0.7 & 8.0±0.3 & 12.1±0.5 & \red{13.2±0.9} \\
 & 20\_vs\_20 & 13.1±0.5 & 4.8±0.6 & 10.5±0.9 & 11.8±0.5 & 9.1±0.5 & 14.0±0.9 & \red{14.8±1.0} \\
 & 20\_vs\_23 & 11.2±0.5 & 3.5±0.6 & 5.6±0.7 & 10.2±0.7 & 7.4±0.7 & 13.0±1.1 & \red{13.3±0.9} \\
\midrule
\multirow{5}{*}{Terran} & 5\_vs\_5 & 10.8±1.4 & 6.4±1.1 & 6.5±0.9 & 6.8±0.6 & 6.9±0.6 & 10.5±1.2 & \red{10.7±1.5} \\
 & 10\_vs\_10 & 10.3±0.3 & 4.6±0.4 & 6.8±0.6 & 8.7±1.4 & 7.6±1.0 & 10.1±0.6 & \red{11.8±0.9} \\
 & 10\_vs\_11 & 9.0±0.7 & 3.6±1.1 & 5.5±0.2 & 5.5±0.9 & 5.9±0.7 & 8.8±1.4 & \red{9.4±0.9} \\
 & 20\_vs\_20 & 10.8±0.8 & 3.9±0.6 & 4.3±0.6 & 8.3±0.3 & 7.3±0.4 & 10.5±0.7 & \red{11.8±0.5} \\
 & 20\_vs\_23 & 7.2±1.0 & 1.2±1.0 & 1.6±0.2 & 5.3±0.5 & 5.1±0.3 & 7.9±0.6 & \red{8.2±0.7} \\
\midrule
\multirow{5}{*}{Zerg} & 5\_vs\_5 & 10.5±2.2 & 6.6±0.2 & 6.7±0.5 & 6.5±0.9 & 7.7±0.9 & 8.9±1.1 & \red{10.7±2.0} \\
 & 10\_vs\_10 & 11.0±0.8 & 7.3±1.0 & 7.2±0.3 & 7.7±1.1 & 7.5±0.8 & \red{11.8±1.6} & 11.5±1.0 \\
 & 10\_vs\_11 & 9.2±1.1 & 7.6±0.9 & 6.7±0.4 & 6.8±1.0 & 6.5±1.0 & 9.5±1.2 & \red{11.0±0.9} \\
 & 20\_vs\_20 & 9.3±0.5 & 3.7±0.4 & 4.7±0.3 & 6.9±0.5 & 6.9±0.8 & 9.2±0.5 & \red{9.4±1.2} \\
 & 20\_vs\_23 & 8.5±0.7 & 3.3±0.3 & 4.1±0.6 & 6.9±0.5 & 5.7±0.4 & 9.8±0.6 & \red{10.5±0.8} \\
\bottomrule
\end{tabular}
\caption{Comparison of average returns for ComaDICE and baselines on SMACv2 tasks.}
\label{SMACv2:return}
\end{table}

\begin{figure}[H]\small
\small
\centering
\showlegend[0.8]{10}{7}
\\
\showinstance[0.17]{0}{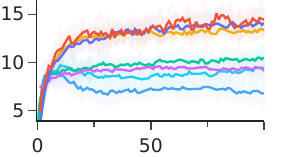}{protoss \\ \emph{5\_vs\_5}}
\showinstance[0.15]{12}{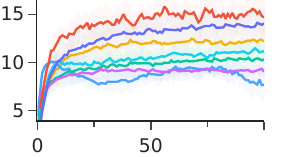}{protoss \\ \emph{10\_vs\_10}}
\showinstance[0.15]{12}{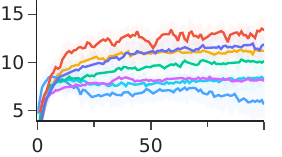}{protoss \\ \emph{10\_vs\_11}}
\showinstance[0.15]{12}{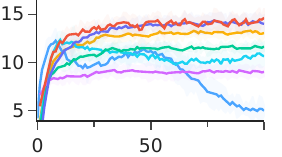}{protoss \\ \emph{20\_vs\_20}}
\showinstance[0.15]{12}{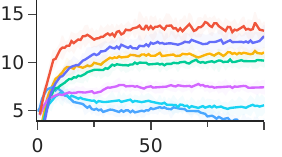}{protoss \\ \emph{20\_vs\_23}}
\\
\showinstance[0.17]{0}{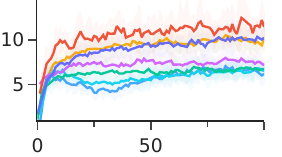}{terran \\ \emph{5\_vs\_5}}
\showinstance[0.15]{12}{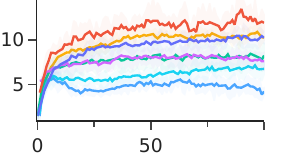}{terran \\ \emph{10\_vs\_10}}
\showinstance[0.15]{12}{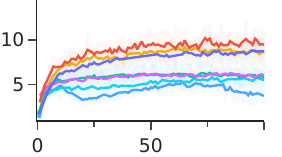}{terran \\ \emph{10\_vs\_11}}
\showinstance[0.15]{12}{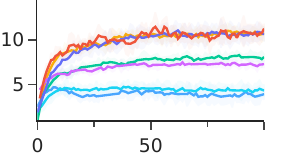}{terran \\ \emph{20\_vs\_20}}
\showinstance[0.15]{12}{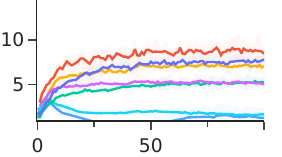}{terran \\ \emph{20\_vs\_23}}
\\
\showinstance[0.17]{0}{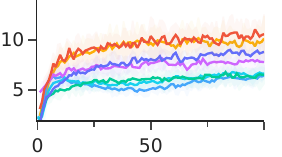}{zerg \\ \emph{5\_vs\_5}}
\showinstance[0.15]{12}{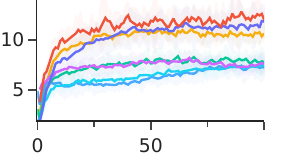}{zerg \\ \emph{10\_vs\_10}}
\showinstance[0.15]{12}{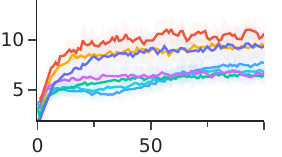}{zerg \\ \emph{10\_vs\_11}}
\showinstance[0.15]{12}{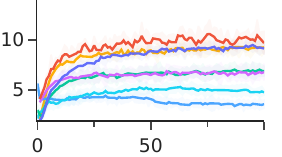}{zerg \\ \emph{20\_vs\_20}}
\showinstance[0.15]{12}{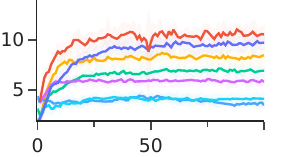}{zerg \\ \emph{20\_vs\_23}}
\\
\caption{Evaluation of SMACv2 tasks comparing the returns achieved by ComaDICE and baselines.}
\label{SMACv2:curves:returns}
\end{figure}

\begin{table}[H]
\small
\centering
\begin{tabular}{cc|cccc}
\toprule
\multicolumn{2}{c|}{\textbf{Instances}} & \textbf{expert} & \textbf{medium} & \textbf{m-replay} & \textbf{m-expert} \\
\midrule
\multirow{7}{*}{Hopper} & {BC} & 209.8±191.1 & 511.9±7.4 & 133.3±53.5 & 155.3±111.5 \\
 & {BCQ} & 77.9±58.0 & 44.6±20.6 & 26.5±24.0 & 54.3±23.7 \\
 & {CQL} & 159.1±313.8 & 401.3±199.9 & 31.4±15.2 & 64.8±123.3 \\
 & {ICQ} & 754.7±806.3 & 501.8±14.0 & 195.4±103.6 & 355.4±373.9 \\
 & {OMAR} & 2.4±1.5 & 21.3±24.9 & 3.3±3.2 & 1.4±0.9 \\
 & {OMIGA} & 859.6±709.5 & 1189.3±544.3 & 774.2±494.3 & 709.0±595.7 \\
 & {ComaDICE} & 2827.7±62.9 & 822.6±66.2 & 906.3±242.1 & 1362.4±522.9 \\
\midrule
\multirow{7}{*}{Ant} & {BC} & 2046.3±6.2 & 1421.1±7.9 & 994.0±20.3 & 1561.7±64.8 \\
 & {BCQ} & 1317.7±286.3 & 1059.6±91.2 & 950.8±48.8 & 1020.9±242.7 \\
 & {CQL} & 1042.4±2021.6 & 533.9±1766.4 & 234.6±1618.3 & 800.2±1621.5 \\
 & {ICQ} & 2050.0±11.9 & 1412.4±10.9 & 1016.7±53.5 & 1590.2±85.6 \\
 & {OMAR} & 312.5±297.5 & -1710.0±1589.0 & -2014.2±844.7 & -2992.8±7.0 \\
 & {OMIGA} & 2055.5±1.6 & 1418.4±5.4 & 1105.1±88.9 & 1720.3±110.6 \\
 & {ComaDICE} & 2056.9±5.9 & 1425.0±2.9 & 1122.9±61.0 & 1813.9±68.4 \\
\midrule
\multirow{7}{*}{\shortstack{Half\\Cheetah}} & {BC} & 3251.2±386.8 & 2280.3±178.2 & 1886.2±390.8 & 2451.9±783.0 \\
 & {BCQ} & 2992.7±629.7 & 2590.5±1110.4 & -333.6±152.1 & 3543.7±780.9 \\
 & {CQL} & 1189.5±1034.5 & 1011.3±1016.9 & 1998.7±693.9 & 1194.2±1081.0 \\
 & {ICQ} & 2955.9±459.2 & 2549.3±96.3 & 1922.4±612.9 & 2834.0±420.3 \\
 & {OMAR} & -206.7±161.1 & -265.7±147.0 & -235.4±154.9 & -253.8±63.9 \\
 & {OMIGA} & 3383.6±552.7 & 3608.1±237.4 & 2504.7±83.5 & 2948.5±518.9 \\
 & {ComaDICE} & 4082.9±45.7 & 2664.7±54.2 & 2855.0±242.2 & 3889.7±81.6 \\
\bottomrule
\end{tabular}
\caption{Comparison of average returns for ComaDICE and baselines on MaMujoco benchmarks.}
\label{MaMujoco:return}
\end{table}

\begin{figure}[H]\small
\small
\centering
\showlegend[0.8]{10}{7}
\\
\showinstance[0.17]{0}{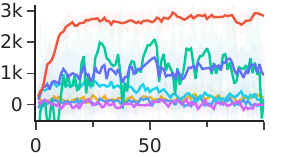}{Hopper \\ \emph{expert}}
\showinstance[0.15]{12}{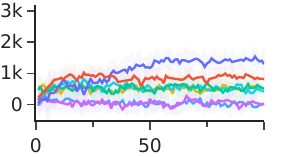}{Hopper \\ \emph{medium}}
\showinstance[0.15]{12}{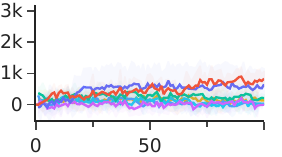}{Hopper \\ \emph{m-replay}}
\showinstance[0.15]{12}{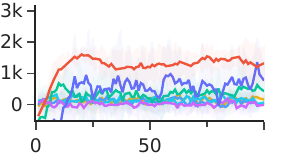}{Hopper \\ \emph{m-expert}}
\\
\showinstance[0.17]{0}{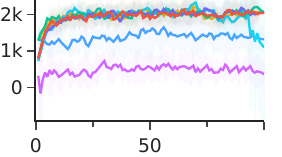}{Ant \\ \emph{expert}}
\showinstance[0.15]{12}{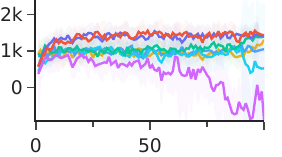}{Ant \\ \emph{medium}}
\showinstance[0.15]{12}{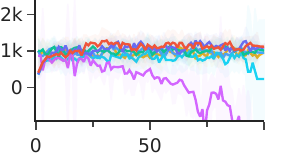}{Ant \\ \emph{m-replay}}
\showinstance[0.15]{12}{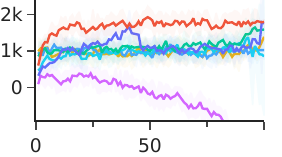}{Ant \\ \emph{m-expert}}
\\
\showinstance[0.17]{0}{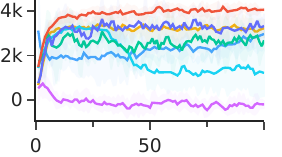}{HalfCheetah \\ \emph{expert}}
\showinstance[0.15]{12}{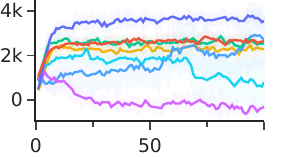}{HalfCheetah \\ \emph{medium}}
\showinstance[0.15]{12}{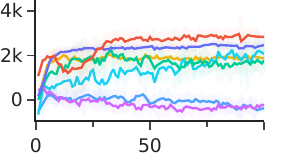}{HalfCheetah \\ \emph{m-replay}}
\showinstance[0.15]{12}{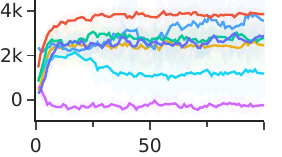}{HalfCheetah \\ \emph{m-expert}}
\\
\caption{Evaluation of MaMujoco tasks comparing the returns achieved by ComaDICE and baselines.}
\label{MaMujoco:curves:returns}
\end{figure}

\subsubsection{Winrates}

In this section, we analyze the winrates of our ComaDICE algorithm across various multi-agent reinforcement learning scenarios. Winrates are crucial in competitive environments like SMACv1 and SMACv2, as they measure the algorithm's success against other agents. Our results demonstrate that ComaDICE consistently achieves higher winrates compared to baseline methods. Notably, ComaDICE performs well across both simple and complex tasks, reflecting its robustness and adaptability. As shown in Tables \ref{SMACv1:winrate} and \ref{SMACv2:winrate}, as well as Figures \ref{SMACv1:curves:winrates} and \ref{SMACv2:curves:winrates}, ComaDICE not only excels in average winrates but also exhibits lower variance, indicating stable performance across different trials. These findings highlight ComaDICE's ability to effectively manage distributional shifts.

\begin{table}[H]
\small
\centering
\resizebox{\textwidth}{!}{
\begin{tabular}{cc|ccccccc}
\toprule
\multicolumn{2}{c|}{\textbf{Instances}} & \textbf{BC} & \textbf{BCQ} & \textbf{CQL} & \textbf{ICQ} & \textbf{OMAR} & \textbf{OMIGA} & \textbf{ComaDICE} \\
\midrule
\multirow{3}{*}{2c\_vs\_64zg} & poor & 0.0±0.0 & 0.0±0.0 & 0.0±0.0 & 0.0±0.0 & 0.0±0.0 & 0.0±0.0 & \red{0.6±1.3} \\
 & medium & 1.9±1.5 & 2.5±3.6 & 2.5±3.6 & 1.9±1.5 & 1.2±1.5 & 6.2±5.6 & \red{8.8±7.0} \\
 & good & 31.2±9.9 & 35.6±8.8 & 44.4±13.0 & 28.7±4.6 & 28.7±9.1 & 40.6±9.5 & \red{55.0±1.5} \\
\midrule
\multirow{3}{*}{5m\_vs\_6m} & poor & 2.5±1.3 & 1.2±1.5 & 1.2±1.5 & 1.2±1.5 & 0.6±1.2 & \red{6.9±1.2} & 4.4±4.2 \\
 & medium & 1.9±1.5 & 1.2±1.5 & 2.5±1.2 & 1.2±1.5 & 0.6±1.2 & 2.5±3.1 & \red{7.5±2.5} \\
 & good & 2.5±2.3 & 1.9±2.5 & 1.9±1.5 & 3.8±2.3 & 3.8±1.2 & 6.9±1.2 & \red{8.1±3.2} \\
\midrule
\multirow{3}{*}{6h\_vs\_8z} & poor & 0.0±0.0 & 0.0±0.0 & 0.0±0.0 & 0.0±0.0 & 0.0±0.0 & 0.0±0.0 & \red{1.9±3.8} \\
 & medium & 1.9±1.5 & 1.9±1.5 & 1.9±1.5 & 2.5±1.2 & 1.9±1.5 & 1.2±1.5 & \red{3.1±2.0} \\
 & good & 8.8±1.2 & 8.8±3.6 & 7.5±1.5 & 9.4±2.0 & 0.6±1.3 & 5.6±3.6 & \red{11.2±5.4} \\
\midrule
\multirow{3}{*}{corridor} & poor & 0.0±0.0 & 0.0±0.0 & 0.0±0.0 & \red{0.6±1.3} & 0.0±0.0 & 0.0±0.0 & \red{0.6±1.3} \\
 & medium & 15.0±2.3 & 23.1±1.5 & 14.4±1.5 & 22.5±3.1 & 11.9±2.3 & 23.8±5.1 & \red{27.3±3.4} \\
 & good & 30.6±4.1 & 42.5±6.4 & 5.6±1.2 & 42.5±6.4 & 3.1±0.0 & 41.9±6.4 & \red{48.8±2.5} \\
\bottomrule
\end{tabular}
}
\caption{Comparison of average winrates for ComaDICE and baselines on SMACv1 benchmarks.}
\label{SMACv1:winrate}
\end{table}

\begin{figure}[H]\small
\small
\centering
\showlegend[0.8]{10}{7}
\\
\showinstance[0.17]{0}{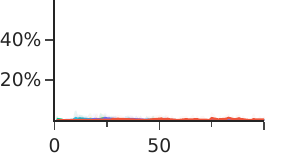}{2c\_vs\_64zg \\ \emph{poor}}
\showinstance[0.14]{20}{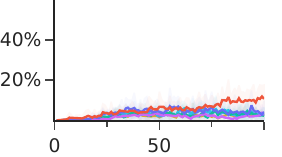}{2c\_vs\_64zg \\ \emph{medium}}
\showinstance[0.14]{20}{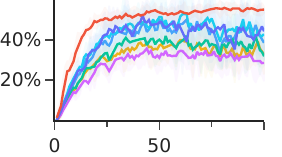}{2c\_vs\_64zg \\ \emph{good}}
\showinstance[0.17]{0}{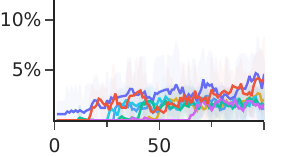}{5m\_vs\_6m \\ \emph{poor}}
\showinstance[0.14]{20}{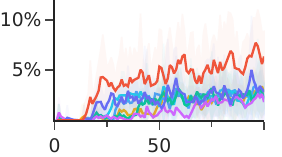}{5m\_vs\_6m \\ \emph{medium}}
\showinstance[0.14]{20}{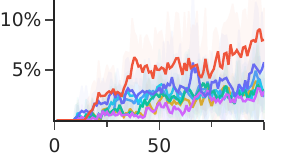}{5m\_vs\_6m \\ \emph{good}}
\\
\showinstance[0.17]{0}{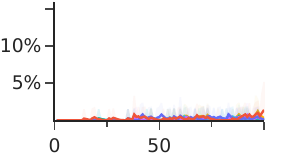}{6h\_vs\_8z \\ \emph{poor}}
\showinstance[0.14]{20}{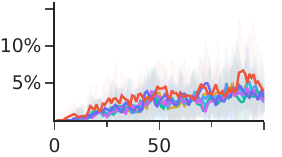}{6h\_vs\_8z \\ \emph{medium}}
\showinstance[0.14]{20}{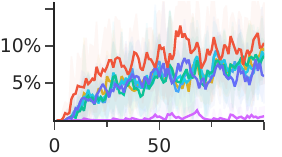}{6h\_vs\_8z \\ \emph{good}}
\showinstance[0.17]{0}{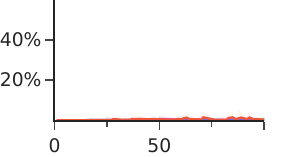}{corridor \\ \emph{poor}}
\showinstance[0.14]{20}{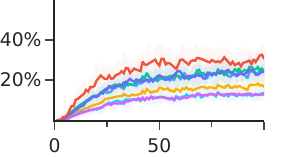}{corridor \\ \emph{medium}}
\showinstance[0.14]{20}{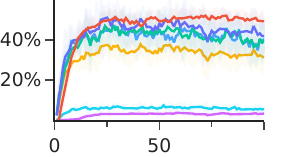}{corridor \\ \emph{good}}
\\
\caption{Evaluation of SMACv1 tasks comparing the winrates achieved by ComaDICE and baselines.}
\label{SMACv1:curves:winrates}
\end{figure}

\begin{table}[H]
\small
\centering
\resizebox{\textwidth}{!}{
\begin{tabular}{cc|ccccccc}
\toprule
\multicolumn{2}{c|}{\textbf{Instances}} & \textbf{BC} & \textbf{BCQ} & \textbf{CQL} & \textbf{ICQ} & \textbf{OMAR} & \textbf{OMIGA} & \textbf{ComaDICE} \\
\midrule
\multirow{5}{*}{Protoss} & 5\_vs\_5 & 36.9±8.7 & 16.2±2.3 & 10.0±4.1 & 36.9±9.1 & 21.2±4.1 & 33.1±5.4 & \red{46.2±6.1} \\
 & 10\_vs\_10 & 36.2±10.6 & 9.4±5.6 & 26.2±7.6 & 28.1±6.6 & 13.8±7.0 & 40.0±10.7 & \red{50.6±8.7} \\
 & 10\_vs\_11 & 19.4±4.6 & 10.0±4.1 & 10.6±5.4 & 12.5±4.4 & 12.5±3.4 & 16.2±6.1 & \red{20.0±4.2} \\
 & 20\_vs\_20 & 37.5±4.4 & 6.2±2.0 & 11.9±4.1 & 32.5±8.1 & 23.8±2.5 & 36.2±5.1 & \red{47.5±7.8} \\
 & 20\_vs\_23 & 13.8±1.5 & 1.2±1.5 & 0.0±0.0 & 12.5±5.6 & 11.2±7.8 & 12.5±8.1 & \red{13.8±5.8} \\
\midrule
\multirow{5}{*}{Terran} & 5\_vs\_5 & 30.0±4.2 & 12.5±6.2 & 9.4±7.9 & 23.1±5.8 & 14.4±4.7 & 28.1±4.4 & \red{30.6±8.2} \\
 & 10\_vs\_10 & 29.4±5.8 & 6.9±6.1 & 9.4±5.6 & 16.9±5.8 & 15.0±4.6 & 29.4±3.2 & \red{32.5±5.8} \\
 & 10\_vs\_11 & 16.2±3.6 & 3.8±4.6 & 7.5±6.4 & 5.0±4.2 & 9.4±5.6 & 12.5±5.2 & \red{19.4±5.4} \\
 & 20\_vs\_20 & 26.2±10.4 & 5.0±3.2 & 10.6±4.2 & 15.6±3.4 & 7.5±7.3 & 21.9±4.4 & \red{29.4±3.8} \\
 & 20\_vs\_23 & 4.4±4.2 & 0.0±0.0 & 0.0±0.0 & 7.5±6.1 & 5.0±4.2 & 4.4±2.5 & \red{9.4±5.2} \\
\midrule
\multirow{5}{*}{Zerg} & 5\_vs\_5 & 26.9±10.0 & 14.4±4.2 & 14.4±5.8 & 18.8±7.1 & 13.8±6.1 & 21.9±5.9 & \red{31.2±7.7} \\
 & 10\_vs\_10 & 25.0±2.8 & 5.6±4.6 & 5.6±4.6 & 15.6±7.4 & 19.4±2.3 & 23.8±6.4 & \red{33.8±11.8} \\
 & 10\_vs\_11 & 13.8±4.7 & 9.4±5.2 & 6.2±4.4 & 10.6±6.7 & 10.6±3.8 & 13.8±6.7 & \red{19.4±3.6} \\
 & 20\_vs\_20 & 8.1±1.5 & 2.5±1.2 & 1.2±1.5 & 10.0±7.8 & \red{12.5±4.4} & 10.0±2.3 & 9.4±6.2 \\
 & 20\_vs\_23 & 7.5±3.2 & 0.6±1.3 & 1.2±1.5 & 7.5±3.2 & 3.8±2.3 & 4.4±4.2 & \red{11.2±4.2} \\
\bottomrule
\end{tabular}
}
\caption{Comparison of average winrates for ComaDICE and baselines on SMACv2 benchmarks.}
\label{SMACv2:winrate}
\end{table}

\begin{figure}[H]
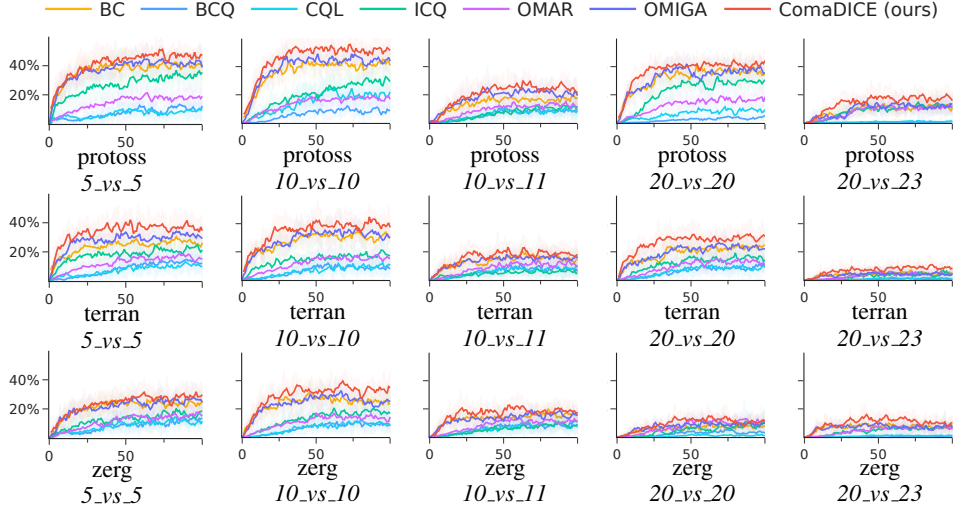
\small
\small
\centering
\showlegend[0.8]{10}{7}
\\
\showinstance[0.17]{0}{protoss_5_vs_5_winrates}{protoss \\ \emph{5\_vs\_5}}
\showinstance[0.14]{20}{protoss_10_vs_10_winrates}{protoss \\ \emph{10\_vs\_10}}
\showinstance[0.14]{20}{protoss_10_vs_11_winrates}{protoss \\ \emph{10\_vs\_11}}
\showinstance[0.14]{20}{protoss_20_vs_20_winrates}{protoss \\ \emph{20\_vs\_20}}
\showinstance[0.14]{20}{protoss_20_vs_23_winrates}{protoss \\ \emph{20\_vs\_23}}
\\
\showinstance[0.17]{0}{terran_5_vs_5_winrates}{terran \\ \emph{5\_vs\_5}}
\showinstance[0.14]{20}{terran_10_vs_10_winrates}{terran \\ \emph{10\_vs\_10}}
\showinstance[0.14]{20}{terran_10_vs_11_winrates}{terran \\ \emph{10\_vs\_11}}
\showinstance[0.14]{20}{terran_20_vs_20_winrates}{terran \\ \emph{20\_vs\_20}}
\showinstance[0.14]{20}{terran_20_vs_23_winrates}{terran \\ \emph{20\_vs\_23}}
\\
\showinstance[0.17]{0}{zerg_5_vs_5_winrates}{zerg \\ \emph{5\_vs\_5}}
\showinstance[0.14]{20}{zerg_10_vs_10_winrates}{zerg \\ \emph{10\_vs\_10}}
\showinstance[0.14]{20}{zerg_10_vs_11_winrates}{zerg \\ \emph{10\_vs\_11}}
\showinstance[0.14]{20}{zerg_20_vs_20_winrates}{zerg \\ \emph{20\_vs\_20}}
\showinstance[0.14]{20}{zerg_20_vs_23_winrates}{zerg \\ \emph{20\_vs\_23}}
\\
\caption{Evaluation of SMACv2 tasks comparing the winrates achieved by ComaDICE and baselines.}
\label{SMACv2:curves:winrates}
\end{figure}

\subsection{Ablation Study: Different Values of Alpha}
We provide more experimental details  for ablation study assessing  the impact of varying the regularization parameter alpha ($\alpha$) on the performance of our ComaDICE.

\subsubsection{Returns}
Our results, in Tables \ref{SMACv1:alpha:return}, \ref{SMACv2:alpha:return}, and \ref{MaMujoco:alpha:return}, show that the performance of ComaDICE is sensitive to the choice of $\alpha$. Lower values of $\alpha$ tend to prioritize imitation learning, leading to suboptimal performance in terms of returns, whereas higher values facilitate better adaptation to the offline data, achieving superior returns. Notably, an $\alpha$ value of 10 consistently yielded the best results across most tasks, indicating an optimal balance between exploration and exploitation in offline settings. This ablation study underscores the importance of selecting an appropriate $\alpha$ to enhance the algorithm's robustness and effectiveness in handling distributional shifts in offline multi-agent reinforcement learning scenarios.

\begin{table}[H]
\small
\centering
\begin{tabular}{cc|ccccc}
\toprule
\multicolumn{2}{c|}{\textbf{Instances}} & $\alpha=0.01$ & $\alpha=0.1$ & $\alpha=1$ & $\alpha=10$ & $\alpha=100$ \\
\midrule
\multirow{3}{*}{2c\_vs\_64zg} & poor & 10.6±0.5 & 11.1±0.4 & 11.1±0.1 & 12.1±0.5 & 11.8±0.2 \\
 & medium & 9.6±0.5 & 13.1±0.8 & 12.5±2.4 & 16.3±0.7 & 16.0±0.3 \\
 & good & 11.1±1.4 & 9.6±2.7 & 17.4±0.5 & 20.3±0.1 & 19.9±0.1 \\
\midrule
\multirow{3}{*}{5m\_vs\_6m} & poor & 5.7±0.1 & 5.1±0.3 & 7.1±0.7 & 8.1±0.5 & 7.7±0.3 \\
 & medium & 5.6±0.1 & 5.3±0.2 & 7.8±0.8 & 8.7±0.4 & 8.5±0.7 \\
 & good & 5.7±0.1 & 5.7±0.2 & 7.8±0.5 & 8.7±0.5 & 8.8±0.8 \\
\midrule
\multirow{3}{*}{6h\_vs\_8z} & poor & 8.5±0.2 & 9.6±0.3 & 10.0±0.3 & 11.4±0.6 & 10.7±0.4 \\
 & medium & 8.5±0.6 & 10.5±0.8 & 10.7±0.5 & 12.8±0.2 & 12.3±0.3 \\
 & good & 7.9±0.1 & 9.5±0.6 & 11.3±0.6 & 13.1±0.5 & 12.8±0.4 \\
\midrule
\multirow{3}{*}{corridor} & poor & 2.1±0.4 & 3.7±1.0 & 6.1±0.8 & 6.4±0.5 & 5.0±1.1 \\
 & medium & 1.7±1.0 & 2.2±1.7 & 11.3±0.3 & 12.9±0.6 & 13.3±0.1 \\
 & good & 4.7±2.4 & 3.8±5.0 & 15.7±0.3 & 18.0±0.1 & 17.4±0.1 \\
\bottomrule
\end{tabular}
\caption{Impact of alpha on returns for ComaDICE and baselines in SMACv1.}
\label{SMACv1:alpha:return}
\end{table}

\begin{table}[H]
\small
\centering
\begin{tabular}{cc|ccccc}
\toprule
\multicolumn{2}{c|}{\textbf{Instances}} & $\alpha=0.01$ & $\alpha=0.1$ & $\alpha=1$ & $\alpha=10$ & $\alpha=100$ \\
\midrule
\multirow{5}{*}{Protoss} & 5\_vs\_5 & 12.2±1.0 & 13.1±1.3 & 13.2±1.1 & 14.4±1.1 & 14.0±2.0 \\
 & 10\_vs\_10 & 12.8±0.9 & 14.0±0.8 & 13.4±1.2 & 14.6±1.8 & 14.1±1.3 \\
 & 10\_vs\_11 & 9.9±1.1 & 11.1±0.8 & 11.3±1.2 & 13.2±0.9 & 12.2±1.1 \\
 & 20\_vs\_20 & 10.3±0.5 & 11.1±1.0 & 12.2±0.9 & 14.8±1.0 & 13.2±0.4 \\
 & 20\_vs\_23 & 8.0±2.3 & 11.2±1.2 & 11.7±0.6 & 13.3±0.9 & 13.2±0.5 \\
\midrule
\multirow{5}{*}{Terran} & 5\_vs\_5 & 11.1±1.8 & 10.1±1.2 & 9.0±1.0 & 10.7±1.5 & 12.6±1.9 \\
 & 10\_vs\_10 & 8.5±0.8 & 10.3±0.7 & 10.4±1.1 & 11.8±0.9 & 11.8±1.7 \\
 & 10\_vs\_11 & 7.5±0.7 & 8.6±2.1 & 8.5±1.6 & 9.4±0.9 & 9.6±0.9 \\
 & 20\_vs\_20 & 6.2±1.1 & 6.4±1.7 & 9.1±0.7 & 11.8±0.5 & 9.3±0.6 \\
 & 20\_vs\_23 & 5.5±1.1 & 6.5±1.6 & 6.5±0.8 & 8.2±0.7 & 8.2±0.4 \\
\midrule
\multirow{5}{*}{Zerg} & 5\_vs\_5 & 7.9±0.6 & 9.3±0.9 & 10.5±1.4 & 10.7±2.0 & 10.4±1.2 \\
 & 10\_vs\_10 & 10.9±1.5 & 11.4±1.5 & 11.8±0.7 & 11.5±1.0 & 10.9±2.2 \\
 & 10\_vs\_11 & 10.1±2.5 & 9.1±1.2 & 10.0±1.2 & 11.0±0.9 & 9.8±0.8 \\
 & 20\_vs\_20 & 8.0±0.5 & 9.2±1.3 & 9.2±1.0 & 9.4±1.2 & 10.5±0.9 \\
 & 20\_vs\_23 & 9.1±1.1 & 10.0±0.7 & 10.4±0.6 & 10.5±0.8 & 10.1±0.7 \\
\bottomrule
\end{tabular}
\caption{Impact of alpha on returns for ComaDICE and baselines in SMACv2.}
\label{SMACv2:alpha:return}
\end{table}

\begin{table}[H]
\small
\centering
\resizebox{\textwidth}{!}{
\begin{tabular}{cc|ccccc}
\toprule
\multicolumn{2}{c|}{\textbf{Instances}} & $\alpha=0.01$ & $\alpha=0.1$ & $\alpha=1$ & $\alpha=10$ & $\alpha=100$ \\
\midrule
\multirow{4}{*}{Hopper} & expert & 147.3±67.9 & 107.9±65.5 & 545.7±820.6 & 2827.7±62.9 & 2690.7±58.6 \\
 & medium & 149.6±96.8 & 107.5±66.9 & 244.7±267.5 & 822.6±66.2 & 807.5±122.2 \\
 & m-replay & 165.6±104.1 & 109.6±38.7 & 155.6±61.6 & 906.3±242.1 & 186.5±16.8 \\
 & m-expert & 119.1±77.1 & 95.6±69.5 & 58.8±26.1 & 1362.4±522.9 & 1358.4±595.1 \\
\midrule
\multirow{4}{*}{Ant} & expert & 1016.4±196.5 & 1179.0±273.7 & 1927.7±174.1 & 2056.9±5.9 & 1950.0±3.3 \\
 & medium & 907.3±32.2 & 1000.0±90.4 & 1424.3±3.1 & 1425.0±2.9 & 1354.6±2.5 \\
 & m-replay & 969.1±21.9 & 978.4±39.6 & 944.6±28.9 & 1122.9±61.0 & 1072.1±41.4 \\
 & m-expert & 915.8±364.1 & 1132.9±282.2 & 738.5±250.2 & 1813.9±68.4 & 1559.6±86.8 \\
\midrule
\multirow{4}{*}{\shortstack{Half\\Cheetah}} & expert & 1068.9±635.2 & 935.2±905.9 & 3637.0±80.9 & 4082.9±45.7 & 3843.7±149.4 \\
 & medium & 575.9±724.8 & 445.2±403.9 & 2690.0±92.4 & 2664.7±54.2 & 2523.4±59.0 \\
 & m-replay & 412.3±310.5 & 233.5±270.1 & 861.6±173.5 & 2855.0±242.2 & 2557.4±241.5 \\
 & m-expert & -107.5±298.1 & -275.9±544.5 & 1136.9±1608.3 & 3889.7±81.6 & 3605.6±70.4 \\
\bottomrule
\end{tabular}
}
\caption{Impact of alpha on returns for ComaDICE and baselines in MaMujoco.}
\label{MaMujoco:alpha:return}
\end{table}

\begin{figure}[H]\small
\small
\centering
\showinstance[0.17]{0}{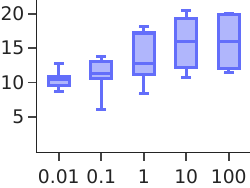}{2c\_vs\_64zg}
\showinstance[0.14]{12}{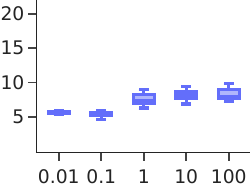}{5m\_vs\_6m}
\showinstance[0.14]{12}{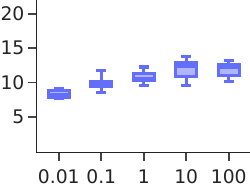}{6h\_vs\_8z}
\showinstance[0.14]{12}{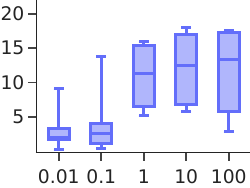}{corridor}
\vspace{8pt}
\\
\showinstance[0.17]{0}{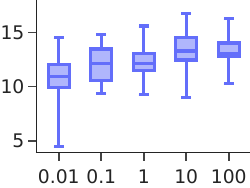}{Protoss}
\showinstance[0.14]{12}{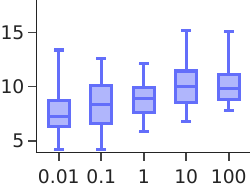}{Terran}
\showinstance[0.14]{12}{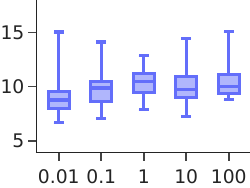}{Zerg}
\showinstance[0.17]{0}{Hopper-v2_box_returns}{Hopper}
\showinstance[0.14]{12}{Ant-v2_box_returns}{Ant}
\showinstance[0.14]{12}{HalfCheetah-v2_box_returns}{HalfCheetah}
\caption{Impact of alpha on returns for ComaDICE and baselines.}
\label{boxes:returns}
\end{figure}

\subsubsection{Winrates}

In the A.4.2 section of the appendix, we investigate the impact of varying $\alpha$ on winrates across different multi-agent reinforcement learning environments. We observe that an intermediate $\alpha$ value of 10 consistently yields optimal results, suggesting it strikes an effective balance between conservative policy adherence and exploration of the offline dataset. This section underscores the importance of fine-tuning $\alpha$ to enhance the robustness and efficacy of the ComaDICE algorithm in managing distributional shifts within competitive multi-agent settings.

\begin{table}[H]
\small
\centering
\begin{tabular}{cc|ccccc}
\toprule
\multicolumn{2}{c|}{\textbf{Instances}} & $\alpha=0.01$ & $\alpha=0.1$ & $\alpha=1$ & $\alpha=10$ & $\alpha=100$ \\
\midrule
\multirow{3}{*}{2c\_vs\_64zg} & poor & 0.0±0.0 & 0.0±0.0 & 0.0±0.0 & 0.6±1.3 & 0.6±1.3 \\
 & medium & 0.0±0.0 & 1.9±3.8 & 5.0±5.1 & 8.8±7.0 & 8.8±4.6 \\
 & good & 0.6±1.2 & 0.0±0.0 & 40.6±4.0 & 55.0±1.5 & 51.9±1.5 \\
\midrule
\multirow{3}{*}{5m\_vs\_6m} & poor & 0.0±0.0 & 0.0±0.0 & 4.4±4.7 & 4.4±4.2 & 1.9±1.5 \\
 & medium & 0.0±0.0 & 0.0±0.0 & 8.1±6.4 & 7.5±2.5 & 7.5±3.8 \\
 & good & 0.0±0.0 & 0.0±0.0 & 6.2±4.4 & 8.1±3.2 & 10.0±6.1 \\
\midrule
\multirow{3}{*}{6h\_vs\_8z} & poor & 0.0±0.0 & 0.0±0.0 & 1.9±3.8 & 1.9±3.8 & 0.6±1.3 \\
 & medium & 0.0±0.0 & 0.6±1.3 & 1.9±1.5 & 3.1±2.0 & 3.1±2.0 \\
 & good & 0.0±0.0 & 0.0±0.0 & 7.5±5.8 & 11.2±5.4 & 7.5±7.3 \\
\midrule
\multirow{3}{*}{corridor} & poor & 0.0±0.0 & 0.6±1.2 & 0.0±0.0 & 0.6±1.3 & 1.2±1.5 \\
 & medium & 0.0±0.0 & 0.0±0.0 & 30.0±5.1 & 27.3±3.4 & 34.4±2.8 \\
 & good & 0.0±0.0 & 4.4±8.8 & 48.8±4.7 & 48.8±2.5 & 49.4±3.6 \\
\bottomrule
\end{tabular}
\caption{Impact of alpha on winrates for ComaDICE and baselines in SMACv1.}
\end{table}

\begin{table}[H]
\small
\centering
\begin{tabular}{cc|ccccc}
\toprule
\multicolumn{2}{c|}{\textbf{Instances}} & $\alpha=0.01$ & $\alpha=0.1$ & $\alpha=1$ & $\alpha=10$ & $\alpha=100$ \\
\midrule
\multirow{5}{*}{Protoss} & 5\_vs\_5 & 20.6±10.0 & 31.9±6.1 & 50.0±2.8 & 46.2±6.1 & 46.2±8.5 \\
 & 10\_vs\_10 & 19.4±6.1 & 25.0±3.4 & 45.0±11.1 & 50.6±8.7 & 51.2±7.6 \\
 & 10\_vs\_11 & 0.0±0.0 & 6.2±9.7 & 18.8±8.1 & 20.0±4.2 & 29.4±8.3 \\
 & 20\_vs\_20 & 1.2±1.5 & 8.8±7.8 & 28.1±8.6 & 47.5±7.8 & 40.6±6.2 \\
 & 20\_vs\_23 & 0.0±0.0 & 1.9±2.5 & 9.4±6.6 & 13.8±5.8 & 17.5±5.1 \\
\midrule
\multirow{5}{*}{Terran} & 5\_vs\_5 & 25.6±4.6 & 22.5±7.2 & 30.6±4.1 & 30.6±8.2 & 41.2±4.6 \\
 & 10\_vs\_10 & 15.0±8.7 & 28.7±7.2 & 33.8±9.4 & 32.5±5.8 & 43.8±7.1 \\
 & 10\_vs\_11 & 3.8±2.3 & 13.8±9.2 & 14.4±9.2 & 19.4±5.4 & 16.2±10.3 \\
 & 20\_vs\_20 & 0.6±1.2 & 2.5±3.6 & 18.8±2.0 & 29.4±3.8 & 21.9±3.4 \\
 & 20\_vs\_23 & 0.6±1.3 & 2.5±3.6 & 2.5±3.6 & 9.4±5.2 & 6.2±2.0 \\
\midrule
\multirow{5}{*}{Zerg} & 5\_vs\_5 & 10.0±4.6 & 20.0±5.8 & 28.7±4.6 & 31.2±7.7 & 25.0±8.6 \\
 & 10\_vs\_10 & 13.8±9.0 & 20.6±8.3 & 29.4±9.0 & 33.8±11.8 & 31.9±6.7 \\
 & 10\_vs\_11 & 9.4±9.5 & 12.5±6.8 & 16.9±3.2 & 19.4±3.6 & 17.5±9.2 \\
 & 20\_vs\_20 & 0.0±0.0 & 1.9±1.5 & 6.9±6.1 & 9.4±6.2 & 12.5±4.0 \\
 & 20\_vs\_23 & 1.2±1.5 & 3.8±2.3 & 12.5±4.0 & 11.2±4.2 & 11.9±6.1 \\
\bottomrule
\end{tabular}
\caption{Impact of alpha on winrates for ComaDICE and baselines in SMACv2.}
\end{table}

\begin{figure}[H]\small
\small
\centering
\showinstance[0.17]{0}{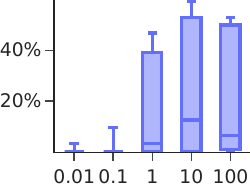}{2c\_vs\_64zg}
\showinstance[0.14]{20}{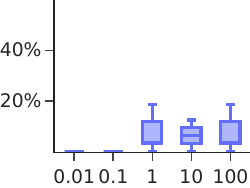}{5m\_vs\_6m}
\showinstance[0.14]{20}{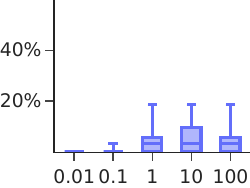}{6h\_vs\_8z}
\showinstance[0.14]{20}{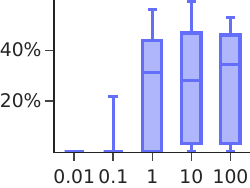}{corridor}
\vspace{8pt}
\\
\showinstance[0.17]{0}{protoss_box_winrates}{protoss}
\showinstance[0.14]{20}{terran_box_winrates}{terran}
\showinstance[0.14]{20}{zerg_box_winrates}{zerg}
\caption{Impact of alpha on winrates for ComaDICE and baselines.}
\label{boxes:Winrates}
\end{figure}

\subsection{Ablation Study: Different Forms of f-divergence}
We conduct an ablation study to examine the effects of different functions of $f$-divergence on the performance of our ComaDICE algorithm across various multi-agent reinforcement learning environments. The study specifically evaluates three types of $f$-divergence: Kullback-Leibler (KL), $\chi^2$, and $\text{Soft-}\chi^2$ .  

\paragraph{KL-Divergence:} This is a well-known measure of how one probability distribution diverges from a second, expected probability distribution. It is defined as:
   \[
   f_{\text{KL}}(x) = x \log x - x + 1
   \]
   The corresponding inverse derivative, which is used in optimization, is:
   \[
   (f'_{\text{KL}})^{-1}(x) = \exp(x - 1)
   \]
   KL-divergence can lead to numerical instability due to the exponential function, especially when the values become large.

\paragraph{\(\chi^2\)-Divergence:} This divergence measures the difference between two probability distributions by considering the square of the differences. It is expressed as:
   \[
   f_{\chi^2}(x) = \frac{1}{2}(x - 1)^2
   \]
   The inverse derivative is:
   \[
   (f'_{\chi^2})^{-1}(x) = x + 1
   \]
   While this function avoids the exponential instability seen in KL-divergence, it may suffer from zero gradients for negative values, which can slow down or halt training.

\paragraph{Soft-\(\chi^2\) Divergence:} This function combines the forms of KL and \(\chi^2\) divergences to mitigate both numerical instability and the dying gradient problem. It is defined piecewise as:
   \[
   f_{\text{Soft-}\chi^2}(x) =
   \begin{cases} 
   x \log x - x + 1 & \text{if } 0 < x < 1 \\
   \frac{1}{2}(x - 1)^2 & \text{if } x \geq 1 
   \end{cases}
   \]
   The inverse derivative is:
   \[
   (f'_{\text{Soft-}\chi^2})^{-1}(x) =
   \begin{cases} 
   \exp(x) & \text{if } x < 0 \\
   x + 1 & \text{if } x \geq 0 
   \end{cases}
   \]
   This choice provides a stable optimization process by maintaining non-zero gradients and avoiding large exponential values, making it suitable for reinforcement learning tasks.

We assess their impact on both returns and winrates in environments such as SMACv1, SMACv2, and MaMujoco. Our results, detailed in Tables \ref{tab:f-divergence-1}-\ref{tab:f-divergence-5}, reveal that the choice of $f$-divergence function significantly influences the algorithm's effectiveness. For instance, the $\text{Soft-}\chi^2$ divergence consistently yields superior returns and competitive winrates across most scenarios, suggesting its robustness in managing distributional shifts in offline settings. Conversely, while $\text{Soft-}\chi^2$ divergence also performs well, particularly in environments with higher complexity, KL divergence shows varying results, indicating its sensitivity to specific task dynamics. This comprehensive analysis underscores the importance of selecting an appropriate $f$-divergence function to optimize ComaDICE's performance in diverse multi-agent reinforcement learning contexts.

\subsubsection{Returns}
\begin{table}[H]
\small
\centering
\begin{tabular}{cc|ccc}
\toprule
\multicolumn{2}{c|}{\textbf{Instances}} & $f_{\chi^2}(x)$ & $f_{\text{KL}}(x)$ & $f_{\text{Soft-}\chi^2}(x)$ \\
\midrule
\multirow{3}{*}{2c\_vs\_64zg} & poor & 11.6±0.2 & 11.1±0.3 & 12.1±0.5 \\
 & medium & 16.1±0.6 & 15.7±0.3 & 16.3±0.7 \\
 & good & 19.7±0.1 & 19.3±0.1 & 20.3±0.1 \\
\midrule
\multirow{3}{*}{5m\_vs\_6m} & poor & 7.8±0.4 & 7.5±0.5 & 8.1±0.5 \\
 & medium & 8.1±0.5 & 7.7±0.4 & 8.7±0.4 \\
 & good & 8.7±0.6 & 8.1±0.4 & 8.7±0.5 \\
\midrule
\multirow{3}{*}{6h\_vs\_8z} & poor & 10.5±0.3 & 10.0±0.2 & 11.4±0.6 \\
 & medium & 12.9±0.4 & 12.4±0.5 & 12.8±0.2 \\
 & good & 12.7±0.4 & 12.4±0.5 & 13.1±0.5 \\
\midrule
\multirow{3}{*}{corridor} & poor & 6.5±0.5 & 6.1±0.4 & 6.4±0.5 \\
 & medium & 12.7±0.7 & 12.0±0.7 & 12.9±0.6 \\
 & good & 17.3±0.1 & 16.9±0.1 & 18.0±0.1 \\
\bottomrule
\end{tabular}
\caption{Impact of $f$-divergence on returns for ComaDICE and baselines in SMACv1.}
\label{tab:f-divergence-1}
\end{table}

\begin{table}[H]
\small
\centering
\begin{tabular}{cc|ccc}
\toprule
\multicolumn{2}{c|}{\textbf{Instances}} & $f_{\chi^2}(x)$ & $f_{\text{KL}}(x)$ & $f_{\text{Soft-}\chi^2}(x)$ \\
\midrule
\multirow{5}{*}{Protoss} & 5\_vs\_5 & 14.6±0.5 & 13.6±0.9 & 14.4±1.1 \\
 & 10\_vs\_10 & 14.7±1.3 & 13.7±1.6 & 14.6±1.8 \\
 & 10\_vs\_11 & 12.8±1.0 & 11.4±1.7 & 13.2±0.9 \\
 & 20\_vs\_20 & 12.7±0.3 & 13.1±0.7 & 14.8±1.0 \\
 & 20\_vs\_23 & 12.4±0.9 & 12.5±0.7 & 13.3±0.9 \\
\midrule
\multirow{5}{*}{Terran} & 5\_vs\_5 & 11.1±1.2 & 12.7±2.0 & 10.7±1.5 \\
 & 10\_vs\_10 & 9.8±0.9 & 10.7±1.3 & 11.8±0.9 \\
 & 10\_vs\_11 & 8.9±0.8 & 8.9±1.0 & 9.4±0.9 \\
 & 20\_vs\_20 & 10.5±0.5 & 10.2±0.7 & 11.8±0.5 \\
 & 20\_vs\_23 & 8.2±0.4 & 7.4±0.7 & 8.2±0.7 \\
\midrule
\multirow{5}{*}{Zerg} & 5\_vs\_5 & 10.0±0.8 & 9.6±1.5 & 10.7±2.0 \\
 & 10\_vs\_10 & 12.4±1.2 & 10.3±1.1 & 11.5±1.0 \\
 & 10\_vs\_11 & 8.9±0.4 & 9.1±1.1 & 11.0±0.9 \\
 & 20\_vs\_20 & 9.0±0.8 & 9.0±0.6 & 9.4±1.2 \\
 & 20\_vs\_23 & 10.2±1.0 & 9.3±0.8 & 10.5±0.8 \\
\bottomrule
\end{tabular}
\caption{Impact of $f$-divergence on returns for ComaDICE and baselines in SMACv2.}
\label{tab:f-divergence-2}
\end{table}

\begin{table}[H]
\small
\centering
\begin{tabular}{cc|ccc}
\toprule
\multicolumn{2}{c|}{\textbf{Instances}} & $f_{\chi^2}(x)$ & $f_{\text{KL}}(x)$ & $f_{\text{Soft-}\chi^2}(x)$ \\
\midrule
\multirow{4}{*}{Hopper} & expert & 2625.0±191.3 & 2018.7±972.0 & 2827.7±62.9 \\
 & medium & 794.4±69.2 & 295.5±227.1 & 822.6±66.2 \\
 & m-replay & 221.3±58.0 & 129.9±55.0 & 906.3±242.1 \\
 & m-expert & 1294.1±520.4 & 105.5±103.9 & 1362.4±522.9 \\
\midrule
\multirow{4}{*}{Ant} & expert & 1945.2±2.8 & 1884.1±27.8 & 2056.9±5.9 \\
 & medium & 1359.2±3.2 & 1346.2±49.8 & 1425.0±2.9 \\
 & m-replay & 1111.1±57.8 & 987.5±33.9 & 1122.9±61.0 \\
 & m-expert & 1655.9±42.8 & 1182.5±405.1 & 1813.9±68.4 \\
\midrule
\multirow{4}{*}{\shortstack{Half\\Cheetah}} & expert & 3860.6±91.5 & 3830.0±88.8 & 4082.9±45.7 \\
 & medium & 2532.3±81.9 & 2347.8±171.8 & 2664.7±54.2 \\
 & m-replay & 2729.9±241.5 & 1258.5±1015.4 & 2855.0±242.2 \\
 & m-expert & 3665.2±74.0 & 3601.0±155.6 & 3889.7±81.6 \\
\bottomrule
\end{tabular}
\caption{Impact of $f$-divergence on returns for ComaDICE and baselines in MaMujoco.}
\label{tab:f-divergence-3}
\end{table}

\subsubsection{Winrates}

\begin{table}[H]
\small
\centering
\begin{tabular}{cc|ccc}
\toprule
\multicolumn{2}{c|}{\textbf{Instances}} & $f_{\chi^2}(x)$ & $f_{\text{KL}}(x)$ & $f_{\text{Soft-}\chi^2}(x)$ \\
\midrule
\multirow{3}{*}{2c\_vs\_64zg} & poor & 0.0±0.0 & 0.0±0.0 & 0.6±1.3 \\
 & medium & 13.1±4.6 & 10.6±3.8 & 8.8±7.0 \\
 & good & 55.6±3.1 & 54.4±1.5 & 55.0±1.5 \\
\midrule
\multirow{3}{*}{5m\_vs\_6m} & poor & 3.8±3.1 & 3.8±3.6 & 4.4±4.2 \\
 & medium & 6.2±2.8 & 5.0±3.8 & 7.5±2.5 \\
 & good & 8.8±3.6 & 6.9±3.1 & 8.1±3.2 \\
\midrule
\multirow{3}{*}{6h\_vs\_8z} & poor & 0.0±0.0 & 0.0±0.0 & 1.9±3.8 \\
 & medium & 5.0±2.5 & 5.0±3.8 & 3.1±2.0 \\
 & good & 9.4±4.4 & 9.4±2.0 & 11.2±5.4 \\
\midrule
\multirow{3}{*}{corridor} & poor & 1.2±1.5 & 1.2±1.5 & 0.6±1.3 \\
 & medium & 31.2±6.2 & 28.1±5.9 & 27.3±3.4 \\
 & good & 49.4±5.4 & 48.1±1.5 & 48.8±2.5 \\
\bottomrule
\end{tabular}
\caption{Impact of $f$-divergence on winrates for ComaDICE and baselines in SMACv1.}
\label{tab:f-divergence-4}
\end{table}

\begin{table}[H]
\small
\centering
\begin{tabular}{cc|ccc}
\toprule
\multicolumn{2}{c|}{\textbf{Instances}} & $f_{\chi^2}(x)$ & $f_{\text{KL}}(x)$ & $f_{\text{Soft-}\chi^2}(x)$ \\
\midrule
\multirow{5}{*}{Protoss} & 5\_vs\_5 & 52.5±4.1 & 46.2±7.2 & 46.2±6.1 \\
 & 10\_vs\_10 & 48.1±7.6 & 55.0±9.8 & 50.6±8.7 \\
 & 10\_vs\_11 & 22.5±8.7 & 20.6±6.1 & 20.0±4.2 \\
 & 20\_vs\_20 & 38.1±2.3 & 41.2±7.8 & 47.5±7.8 \\
 & 20\_vs\_23 & 16.9±4.2 & 15.0±3.6 & 13.8±5.8 \\
\midrule
\multirow{5}{*}{Terran} & 5\_vs\_5 & 41.2±7.2 & 38.8±10.6 & 30.6±8.2 \\
 & 10\_vs\_10 & 30.6±4.1 & 36.2±10.8 & 32.5±5.8 \\
 & 10\_vs\_11 & 15.6±11.5 & 15.0±7.5 & 19.4±5.4 \\
 & 20\_vs\_20 & 33.8±6.4 & 28.7±11.8 & 29.4±3.8 \\
 & 20\_vs\_23 & 5.6±4.1 & 8.1±4.2 & 9.4±5.2 \\
\midrule
\multirow{5}{*}{Zerg} & 5\_vs\_5 & 29.4±9.0 & 33.1±13.3 & 31.2±7.7 \\
 & 10\_vs\_10 & 31.2±7.7 & 26.2±5.1 & 33.8±11.8 \\
 & 10\_vs\_11 & 11.2±1.5 & 16.2±7.2 & 19.4±3.6 \\
 & 20\_vs\_20 & 7.5±3.2 & 11.2±7.0 & 9.4±6.2 \\
 & 20\_vs\_23 & 10.6±3.2 & 10.0±2.3 & 11.2±4.2 \\
\bottomrule
\end{tabular}
\caption{Impact of $f$-divergence on winrates for ComaDICE and baselines in SMACv2.}
\label{tab:f-divergence-5}
\end{table}

\subsection{Ablation Study: Different Types of Mixer Network}

In this section, we explore the impact of using different types of mixer networks within the ComaDICE algorithm. We introduce two  settings for the mixer network within the ComaDICE algorithm: 1-layer and 2-layer settings. The mixer network plays a crucial role in aggregating local value functions into a global value function, which is essential for effective policy optimization in multi-agent reinforcement learning (MARL) settings. By examining various mixer network architectures, we aim to understand how these configurations affect the performance and stability of the ComaDICE algorithm.
The comparisons are presented in Tables \ref{tab:mixer-1}-\ref{tab:mixer-5}, showing both average returns and win rates. The results clearly demonstrate that the 1-layer configuration performs better, providing more stable training outcomes across nearly all tasks. This contradicts the findings in many prior online MARL studies ~\citep{rashid2020monotonic,son2019qtran,wang2020qplex}, which may be due to over-fitting issues in offline learning.

\subsubsection{Returns}
\begin{table}[H]
\small
\centering
\begin{tabular}{c|c|c|c}
\toprule
\multicolumn{2}{c|}{\multirow{2}{*}{\textbf{Instances}}} & \multicolumn{2}{c}{\textbf{ComaDICE} (ours)} \\ \multicolumn{2}{c|}{} & 1-layer & 2-layer \\
\midrule
\multirow{3}{*}{2c\_vs\_64zg} & poor & 12.1±0.5 & 11.5±0.9 \\
 & medium & 16.3±0.7 & 11.2±0.8 \\
 & good & 20.3±0.1 & 9.0±2.2 \\
\midrule
\multirow{3}{*}{5m\_vs\_6m} & poor & 8.1±0.5 & 3.8±1.1 \\
 & medium & 8.7±0.4 & 0.8±0.3 \\
 & good & 8.7±0.5 & 7.7±0.1 \\
\midrule
\multirow{3}{*}{6h\_vs\_8z} & poor & 11.4±0.6 & 10.3±0.3 \\
 & medium & 12.8±0.2 & 9.1±0.6 \\
 & good & 13.1±0.5 & 8.3±0.5 \\
\midrule
\multirow{3}{*}{corridor} & poor & 6.4±0.5 & 1.5±0.7 \\
 & medium & 12.9±0.6 & 3.9±1.7 \\
 & good & 18.0±0.1 & 2.6±2.3 \\
\bottomrule
\end{tabular}
\caption{Average returns for ComaDICE and baselines on SMACv1 with different mixer settings.}
\label{tab:mixer-1}
\end{table}

\begin{table}[H]
\small
\centering
\begin{tabular}{c|c|c|c}
\toprule
\multicolumn{2}{c|}{\multirow{2}{*}{\textbf{Instances}}} & \multicolumn{2}{c}{\textbf{ComaDICE} (ours)} \\ \multicolumn{2}{c|}{} & 1-layer & 2-layer \\
\midrule
\multirow{5}{*}{Protoss} & 5\_vs\_5 & 14.4±1.1 & 10.5±1.4 \\
 & 10\_vs\_10 & 14.6±1.8 & 11.2±1.6 \\
 & 10\_vs\_11 & 13.2±0.9 & 9.5±0.4 \\
 & 20\_vs\_20 & 14.8±1.0 & 9.5±0.9 \\
 & 20\_vs\_23 & 13.3±0.9 & 7.1±2.2 \\
\midrule
\multirow{5}{*}{Terran} & 5\_vs\_5 & 10.7±1.5 & 8.3±0.8 \\
 & 10\_vs\_10 & 11.8±0.9 & 8.8±1.1 \\
 & 10\_vs\_11 & 9.4±0.9 & 6.4±1.2 \\
 & 20\_vs\_20 & 11.8±0.5 & 7.8±0.9 \\
 & 20\_vs\_23 & 8.2±0.7 & 6.6±0.9 \\
\midrule
\multirow{5}{*}{Zerg} & 5\_vs\_5 & 10.7±2.0 & 7.8±1.1 \\
 & 10\_vs\_10 & 11.5±1.0 & 9.7±0.6 \\
 & 10\_vs\_11 & 11.0±0.9 & 7.9±0.7 \\
 & 20\_vs\_20 & 9.4±1.2 & 7.8±0.6 \\
 & 20\_vs\_23 & 10.5±0.8 & 8.0±0.5 \\
\bottomrule
\end{tabular}
\caption{Average returns for ComaDICE and baselines on SMACv2 with different mixer settings.}
\label{tab:mixer-2}
\end{table}

\begin{table}[H]
\small
\centering
\begin{tabular}{c|c|c|c}
\toprule
\multicolumn{2}{c|}{\multirow{2}{*}{\textbf{Instances}}} & \multicolumn{2}{c}{\textbf{ComaDICE} (ours)} \\ \multicolumn{2}{c|}{} & 1-layer & 2-layer \\
\midrule
\multirow{4}{*}{Hopper} & expert & 2827.7±62.9 & 483.7±349.7 \\
 & medium & 822.6±66.2 & 648.4±245.9 \\
 & m-replay & 906.3±242.1 & 441.9±260.8 \\
 & m-expert & 1362.4±522.9 & 402.3±288.2 \\
\midrule
\multirow{4}{*}{Ant} & expert & 2056.9±5.9 & 1583.0±160.4 \\
 & medium & 1425.0±2.9 & 1198.9±53.9 \\
 & m-replay & 1122.9±61.0 & 1041.8±38.4 \\
 & m-expert & 1813.9±68.4 & 1426.6±171.4 \\
\midrule
\multirow{4}{*}{\shortstack{Half\\Cheetah}} & expert & 4082.9±45.7 & 2159.4±658.0 \\
 & medium & 2664.7±54.2 & 2026.7±244.3 \\
 & m-replay & 2855.0±242.2 & 1299.2±196.1 \\
 & m-expert & 3889.7±81.6 & 1336.3±381.9 \\
\bottomrule
\end{tabular}
\caption{Average returns for ComaDICE and baselines on MaMujoco with different mixer settings.}
\label{tab:mixer-3}
\end{table}

\subsubsection{Winrates}
\begin{table}[H]
\small
\centering
\begin{tabular}{c|c|c|c}
\toprule
\multicolumn{2}{c|}{\multirow{2}{*}{\textbf{Instances}}} & \multicolumn{2}{c}{\textbf{ComaDICE} (ours)} \\ \multicolumn{2}{c|}{} & 1-layer & 2-layer \\
\midrule
\multirow{3}{*}{2c\_vs\_64zg} & poor & 0.6±1.3 & 0.0±0.0 \\
 & medium & 8.8±7.0 & 3.8±3.6 \\
 & good & 55.0±1.5 & 19.4±5.0 \\
\midrule
\multirow{3}{*}{5m\_vs\_6m} & poor & 4.4±4.2 & 3.1±0.0 \\
 & medium & 7.5±2.5 & 1.2±1.5 \\
 & good & 8.1±3.2 & 3.1±0.0 \\
\midrule
\multirow{3}{*}{6h\_vs\_8z} & poor & 1.9±3.8 & 0.0±0.0 \\
 & medium & 3.1±2.0 & 0.0±0.0 \\
 & good & 11.2±5.4 & 1.9±2.5 \\
\midrule
\multirow{3}{*}{corridor} & poor & 0.6±1.3 & 0.0±0.0 \\
 & medium & 27.3±3.4 & 11.2±2.5 \\
 & good & 48.8±2.5 & 23.1±8.1 \\
\bottomrule
\end{tabular}
\caption{Average winrates for ComaDICE and baselines on SMACv1 with different mixer settings.}
\label{tab:mixer-4}
\end{table}

\begin{table}[H]
\small
\centering
\begin{tabular}{c|c|c|c}
\toprule
\multicolumn{2}{c|}{\multirow{2}{*}{\textbf{Instances}}} & \multicolumn{2}{c}{\textbf{ComaDICE} (ours)} \\ \multicolumn{2}{c|}{} & 1-layer & 2-layer \\
\midrule
\multirow{5}{*}{Protoss} & 5\_vs\_5 & 46.2±6.1 & 31.9±3.6 \\
 & 10\_vs\_10 & 50.6±8.7 & 32.5±5.8 \\
 & 10\_vs\_11 & 20.0±4.2 & 10.6±7.3 \\
 & 20\_vs\_20 & 47.5±7.8 & 21.9±4.0 \\
 & 20\_vs\_23 & 13.8±5.8 & 6.9±5.4 \\
\midrule
\multirow{5}{*}{Terran} & 5\_vs\_5 & 30.6±8.2 & 25.6±4.6 \\
 & 10\_vs\_10 & 32.5±5.8 & 28.1±3.4 \\
 & 10\_vs\_11 & 19.4±5.4 & 12.5±4.0 \\
 & 20\_vs\_20 & 29.4±3.8 & 11.2±3.2 \\
 & 20\_vs\_23 & 9.4±5.2 & 3.1±2.0 \\
\midrule
\multirow{5}{*}{Zerg} & 5\_vs\_5 & 31.2±7.7 & 20.6±4.7 \\
 & 10\_vs\_10 & 33.8±11.8 & 21.2±7.2 \\
 & 10\_vs\_11 & 19.4±3.6 & 13.1±4.1 \\
 & 20\_vs\_20 & 9.4±6.2 & 5.6±1.3 \\
 & 20\_vs\_23 & 11.2±4.2 & 3.1±3.4 \\
\bottomrule
\end{tabular}
\caption{Average winrates for ComaDICE and baselines on SMACv2 with different mixer settings.}
\label{tab:mixer-5}
\end{table}

\end{document}